%% file: main.tex
\documentclass[10pt,twocolumn,letterpaper]{article}

\usepackage{cvpr}              % To produce the CAMERA-READY version

\definecolor{cvprblue}{rgb}{0.21,0.49,0.74}
\usepackage{multirow}
\usepackage{placeins}
\usepackage{color}
\usepackage{amsthm, amssymb}
\usepackage{float}
\usepackage{tikz}
\usepackage[pagebackref,breaklinks,colorlinks,allcolors=cvprblue]{hyperref}

\def\paperID{6902} % *** Enter the Paper ID here
\def\confName{CVPR}
\def\confYear{2025}

\title{BiM-VFI: Bidirectional Motion Field-Guided Frame Interpolation for Video with Non-uniform Motions} 
\author{Wonyong Seo\\
KAIST\footnotemark[1]\\
{\tt\small wyong0122@kaist.ac.kr}
\and
Jihyong Oh\footnotemark[2]\\
Chung-Ang University\\
{\tt\small jihyongoh@cau.ac.kr}
\and
Munchurl Kim\footnotemark[2]\\
KAIST\footnotemark[1]\\
{\tt\small mkimee@kaist.ac.kr}
\and
\small{\url{https://kaist-viclab.github.io/BiM-VFI_site/}}
}

\begin{document}
\twocolumn[{
    \renewcommand
    \twocolumn[1][]{#1}%
    
    \maketitle
    
    \centering
    
       \includegraphics[width=0.92\textwidth]{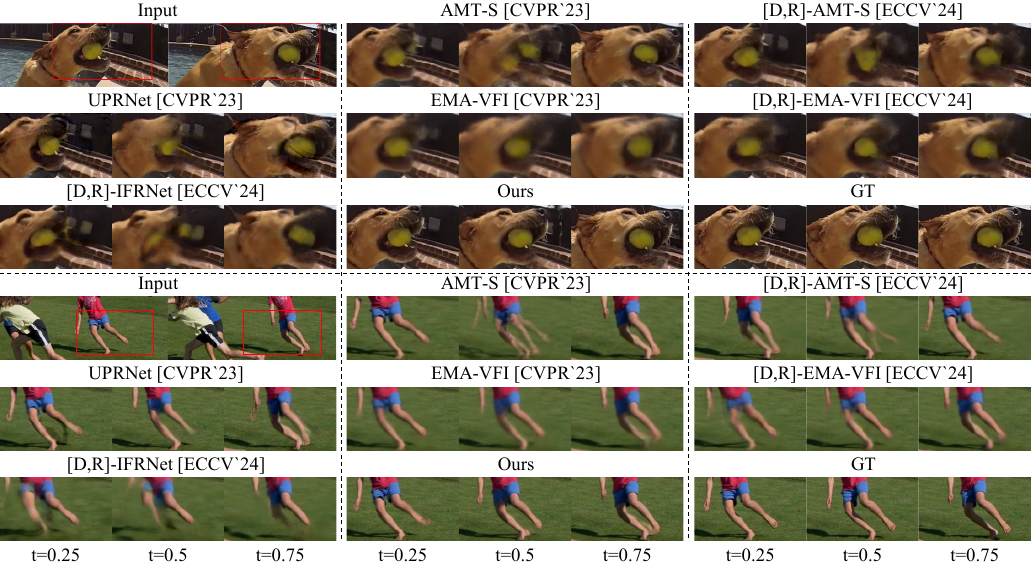}
       \vspace{-0.1cm}
       \captionof{figure}{Qualitative comparison of our proposed BiM-VFI and SOTA models at arbitrary time instances (\(t\) = 0.25, 0.5 and 0.75) for video frame interpolation. The previous SOTA methods yield blurry interpolated frames while our BiM-VFI model generates clear ones.}
       \label{fig:qual_arb}
    
    \vspace{0.5cm}
}]
{
  \renewcommand{\thefootnote}%
    {\fnsymbol{footnote}}
  \footnotetext[1]{Korea Advanced Institute of Science and Technology.}
  \footnotetext[2]{Co-corresponding authors.}
}
\input{sec/0_abstract}    
\input{sec/1_intro}
\input{sec/2_realted_works}

\input{sec/3_methods}
\input{sec/4_results}
\input{sec/5_conclusion}
\input{sec/6_acknowledge}
{
    \small
    \bibliographystyle{ieeenat_fullname}
    \bibliography{main}
}
\input{sec/X_suppl}
\end{document}

%% file: sec/0_abstract.tex
\begin{abstract}
% \vspace{-0.02cm}
Existing Video Frame interpolation (VFI) models tend to suffer from time-to-location ambiguity when trained with video of non-uniform motions, such as accelerating, decelerating, and changing directions, which often yield blurred interpolated frames.
In this paper, we propose (i) a novel motion description map, Bidirectional Motion field (BiM), to effectively describe non-uniform motions;  (ii) a BiM-guided Flow Net (BiMFN)  with Content-Aware Upsampling Network (CAUN) for precise optical flow estimation; and (iii) Knowledge Distillation for VFI-centric Flow supervision (KDVCF) to supervise the motion estimation of VFI model with VFI-centric teacher flows.
The proposed VFI is called a Bidirectional Motion field-guided VFI (BiM-VFI) model.
Extensive experiments show that our BiM-VFI model significantly surpasses the recent state-of-the-art VFI methods by 26\% and 45\% improvements in LPIPS and STLPIPS respectively, yielding interpolated frames with much fewer blurs at arbitrary time instances. 

% For this, we consider an appropriate problem setting for a careful model design, effective training objectives and  supervisions. In this paper, we propose three components needed for VFI training. First, we propose Motion Field, the guidance map enables the distinction of non-uniform motions, such as acceleration, decelaration, and change of directions, included in training triplets. Introducing Motion Field in training time resolve the training ambiguity, thus provides more precise training objective. Second, we design . Finally, we propose novel teacher-student process, which supervise the MFFN with more VFI-centric flows. As a result, our model can improve VFI performance, especially in perceptual perspective. Our codes are available at ~~~~ 
\end{abstract}

%% file: sec/1_intro.tex
\section{Introduction}
\label{sec:introduction}

\begin{figure*}[t]
    \centering
    \includegraphics[width=\textwidth]{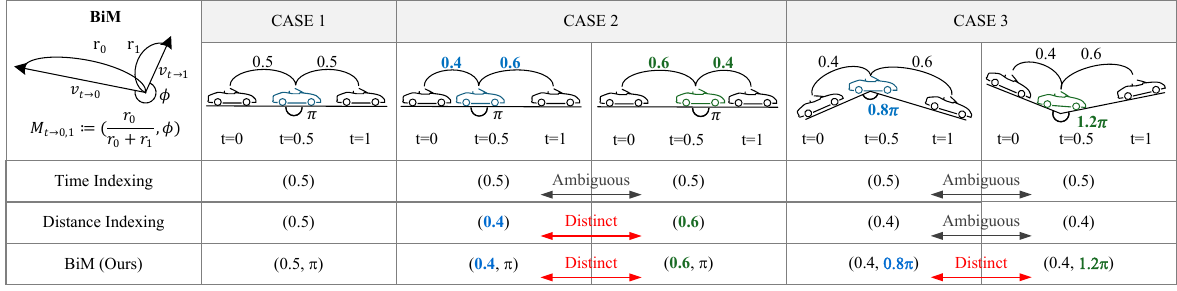}
    % \vspace{-0.6cm}
    \caption{Time-to-location (TTL) ambiguity comparison of motion descriptors (time indexing, distance indexing, our BiM).}
    % \vspace{-0.4cm}
    \label{fig:motion_field}
\end{figure*}

Video frame interpolation (VFI) is a fundamental low-level vision task that aims to synthesize intermediate frames between temporally adjacent input frames.
VFI enables the conversion of low-frame-rate videos into high-frame-rate sequences, enhancing visual fluidity and realism.
VFI has broad applications, including the restoration and enhancement of low frame rate videos, the creation of slow-motion videos~\cite{jiang2018super, xue2019video}, and the improvement of animation workflows in the cartoon industry~\cite{siyao2021deep, plack2023frame, chen2022improving}.

% VFI problems can broadly be divided into two categories: VFI for two consecutive source frames, denoted as VFI2S, and VFI for two or more consecutive source frames, denoted as VFI2S+.
VFI is an ill-posed problem due to the time-to-location (TTL) ambiguity between two input frames~\cite{zhong2023clearer, lu2022video,shi2022video}.
The TTL ambiguity stems from infinitely many possible trajectories between the corresponding pixels of the two source input frames for video sequences with non-uniform motions (CASE1, CASE2, and CASE3 in \cref{fig:motion_field}).
It is well known that TTL ambiguity complicates the prediction of the actual target frame during inference.
However, it also can pose challenges during training.
VFI learning based only on target timesteps \(t\) between the two source input frames can cause VFI networks to learn the average of all the possibilities as final interpolation results which often turn out to be very blurred~\cite{zhong2023clearer}.
% The ill-poseness of VFI becomes more problematic for the VFI2S than the VFI2S+. The adjacent frames of two-center source input frames in the VFI2S+ can help estimating more precise temporal motion trajectories between the corresponding pixels of them at the cost of relatively higher computation complexity.
Solving the TTL ambiguity problem is challenging, especially for inference because of its high ill-posedness.
Instead, we propose an alternative solution not to solve the TTL ambiguity problem at inference time but to resolve ambiguity in the training phase to obtain clean interpolated frames at target time instances.

% Some previous works~\cite{kong2022ifrnet, li2023amt, huang2022real} tried to obtain clean interpolated frames by relying on precise supervision for flow estimation during training. Kong \etal~\cite{kong2022ifrnet} and Li \etal~\cite{li2023amt} utilized pre-trained optical flow models for optical flow supervision. However, using such pretrained models to generate pseudo ground truth flows can lead to performance degradation because the two objectives of supervised optical flow estimation and VFI learning are different, particularly in the areas affected by shadows and blur~\cite{kong2022ifrnet}. Hwang \etal~\cite{huang2022real} obtained `privileged flows' with the help of target frames added to a teacher module during training and supervise a student module's flow estimation.  However, the benefit of such privileged flows is limited as accurate supervision because their teacher module slightly refine the flows obtained by student module with target frame input.

To resolve TTL ambiguity during training, we propose a novel motion description map based on Bidirectional-Motion Fields (BiM), inspired by the distance indexing~\cite{zhong2023clearer}.
The distance indexing relies only on relative distances on the line between two corresponding pixels in source frames, so it is limited in describing the directional changes along motion trajectories, and thus cannot fully resolve the ambiguity.
However, our BiM can fully describe any non-uniform motions including accelerations, decelerations, or directional changes by incorporating both magnitude and angular information of bidirectional flows between a target frame and each of two source frames.
The BiM is used as a description map for VFI learning to generate clean interpolated frames by limiting the solution space of the possible motion trajectories during training.
Also, we design (i) a BiM-guided FlowNet (BiMFN), and (ii) a Content-Aware Upsample Network (CAUN) to precisely estimate the motions based on input BiM.
Lastly, we propose a Knowledge Distillation for VFI-Centric Flow supervision (KDVCF) as a new training strategy with the help of the target frame to generate both BiM as input to student process and accurate flows for student process supervision during the training.
Our proposed VFI model with KDVCF, BiMFN, and CAUN is called Bidirectional Motion field-guided VFI (BiM-VFI).

% The contributions of our work are summarized as:
% \begin{itemize}[leftmargin=2em]
%     \item A novel motion description map, Bidirectional Motion fields (BiM) is proposed to describe non-uniform motions effectively;  
%     \item A Knowledge Distillation for VFI-Centric Flow supervision (KDVCF) is presented to coincide the two purposes of direct supervision on optical flow estimation and photometric reconstruction during training; 
%     \item A BiM-guided FlowNet (BiMFN), with Content-Aware Upsampling Network (CAUN) for precise optical flow estimation.
%      \item Our BiM-VFI model \textit{significantly} surpasses the very recent state-of-the-art VFI methods by 26\% and 45\% improvements in LPIPS~\cite{zhang2018unreasonable} and STLPIPS~\cite{ghildyal2022shift} respectively, yielding interpolated frames with much less blurs at arbitrary target time instances.
% \end{itemize} 

%% file: sec/2_realted_works.tex
\section{Related Works}
\subsection{Video frame interpolation}
\label{subsec:vfi}
VFI methods can be divided into two categories: flow-based and kernel-based approaches.
The flow-based methods~\cite{jin2023unified, li2023amt, zhang2023extracting, huang2022real, kong2022ifrnet} utilize optical flows in interpolating a target frame.
The kernel-based methods construct various types of kernels, such as the adaptive kernels~\cite{niklaus2017ada, niklaus2017sep}, the deformable kernels~\cite{lee2020adacof, xiang2020zooming}, or the attention maps~\cite{shi2022video, lu2022video} to interpolate the target frames by applying these kernels to the source frames.
While flow-based methods can interpolate at any arbitrary time frame, kernel-based methods are limited to interpolating center frames.
Consequently, flow-based methods have dominated the recent VFI works.

Flow-based methods focus on improving the performance of their motion estimators to enhance the interpolation quality.
For the methods~\cite{niklaus2020softmax, hu2022many} employing forward warping~\cite{niklaus2020softmax}, pre-trained optical flow models~\cite{sun2018pwc, huang2022flowformer, teed2020raft} have been directly utilized to estimate the motion to improve the motion estimation accuracy.
For the methods employing backward warping~\cite{jaderberg2015spatial}, flows have to be estimated from the target frame to each of the two source frames, while target frames are unavailable.
Therefore, recent methods tried to design their own motion estimators for accurate flow estimations.
Park \etal ~\cite{park2020bmbc, park2021asymmetric, park2023biformer} and Jin \etal~\cite{jin2023unified} have tailored local cost volumes in a bilateral manner to estimate motions between target frame and each of two source frame. Li \etal~\cite{li2023amt} also adopted `all pair cost volume'~\cite{teed2020raft}, to enhance the motion estimation capabilities of their model.

Recent studies~\cite{kong2022ifrnet, li2023amt} have demonstrated that supervising the flow estimation using pre-trained optical flow models~\cite{hui2018liteflownet} can benefit motion estimation learning, especially in large motions or motions in homogeneous regions, which are not adequately captured by photometric supervision.
However, the pre-trained flow models resulted in degraded performance for VFI in cases of motion in certain regions such as shadows, or blurs, because flows estimated from supervised optical flow models and flows for VFI have distinct roles~\cite{kong2022ifrnet}.
Huang \etal~\cite{huang2022real} introduced the ``privileged block'', which utilizes the target frame to generate a more accurate optical flow from the target to the source frame.
They supervised the flows estimated solely from source frames with these privileged flows to enhance motion estimation performance.
However, the privileged block only consists of a few convolution layers, thus limiting the full utilization of target frames to enhance motion estimation accuracy.
%\jh{Our methods can handle ~~}

\subsection{Non-uniform motions for VFI}
There are studies focused on reducing ambiguity caused by non-uniform motions in VFI.
Xu \etal~\cite{xu2019quadratic} utilized four neighboring consecutive input source frames around each target frame to model motion as a quadratic equation.
Several studies~\cite{lu2022video, shi2022video} also use four neighboring frames with transformer architectures~\cite{dosovitskiy2021an, liu2021swin} to implicitly capture non-uniform motion and interpolate the target frames with self-attention operation.
However, while VFI methods using 4 input frames can reduce TTL ambiguity during training, they still cannot fully resolve it.

Recently, Zhong \etal~\cite{zhong2023clearer} proposed a novel paradigm to address TTL ambiguity during training.
% While natural videos contain complex nonlinear motions including accelerating, decelerating, and directional changes, identical input training source frames can yield different ground truth frames.
% This ambiguity in training data may confuse the VFI models, thus leading to blurred output frames.
Zhong \etal~\cite{zhong2023clearer} introduced a novel motion descriptor called `distance indexing', which describes the relative magnitude between the motion from \(I_{0}\) to \(I_{1}\) and the motion from \(I_{0}\) to \(I_{t}\), using a pixel-wise motion magnitude ratio map.
It has shown that distance indexing can effectively resolve velocity ambiguity but cannot resolve directional ambiguity, because it only includes the motion magnitudes and ignores the directional components.
%\jh{Our methods can handle ~~}

%% file: sec/3_methods.tex
\begin{figure*}[t]
    \centering
    \includegraphics[width=\textwidth]{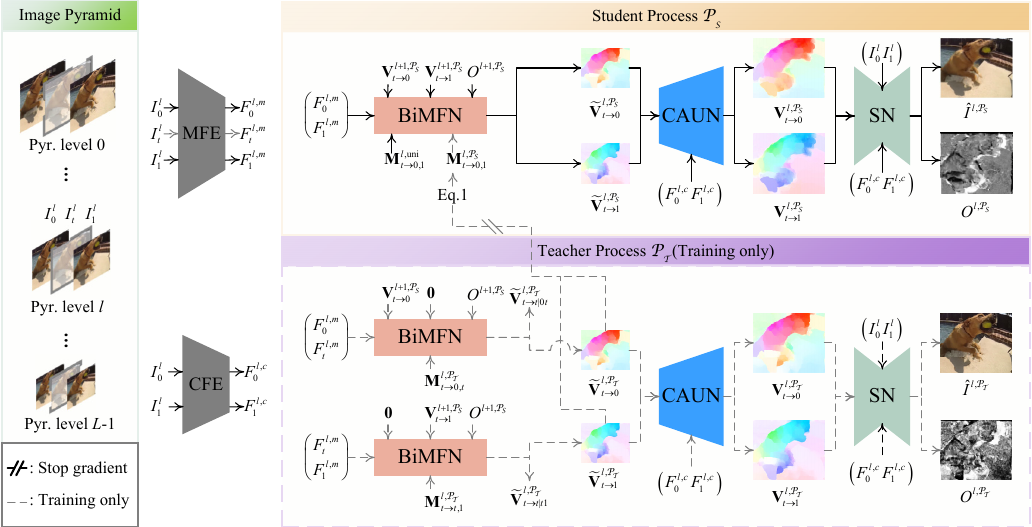}
    % \vspace{-0.5cm}
    \caption{{Our Bidirectional Motion Field-guided VFI (BiM-VFI) with Knowledge Distillation for VFI-Centric Flow supervision (KDVCF).}}
    % \vspace{-0.3cm}
    \label{fig:main}
\end{figure*} 
\section{Proposed Method}

% VFI is a task that estimates a target frame at a desired time between two input source frames.
\cref{fig:main} depicts the overall network architecture of our BiM-VFI based on a recurrent pyramid architecture, operates from \((L-1)\)-th level to \(0\)-th level (where \(L\) is a number pyramid level), and the procedure of proposed KDVCF.
Every pyramid level consists of a pair of student and teacher processes, \(\mathcal{P}_{\mathcal{S}}\) and \(\mathcal{P}_{\mathcal{T}}\), where the weights are shared between the processes as well as across all pyramid levels.
A preceding source frame, a following source frame, and a target frame are denoted as \(I_0\), \(I_1\), and \(I_t\), respectively.

In the BiM-VFI, \(I_0\), \(I_t\), and \(I_1\) are downsampled by a factor \(1/2\) at each hierarchical level.
The Motion Feature Extractor (MFE) extracts motion features \(F_{0}^{l,m}\), \(F_{1}^{l,m}\) and \(F_{t}^{l,m}\) for \(l\)-th level input, while the Context Feature Extractor (CFE) extracts context features \(F_{0}^{l,c}\) and \(F_{1}^{l,c}\).
In the student process, \(F_{0}^{l,m}\) and \(F_{1}^{l,m}\) are fed in Bidirectional Motion Field FlowNet (BiMFN) that outputs bidirectional optical flows, \(\tilde{\textbf{V}}_{t\rightarrow 0}^{l, \mathcal{P}_{\mathcal{S}}}\) and \(\tilde{\textbf{V}}_{t\rightarrow 1}^{l, \mathcal{P}_{\mathcal{S}}}\).
The Content-Aware Upsampling Network (CAUN) takes \(\tilde{\textbf{V}}_{t\rightarrow 0}^{l, \mathcal{P}_{\mathcal{S}}}\), \(\tilde{\textbf{V}}_{t\rightarrow 1}^{l, \mathcal{P}_{\mathcal{S}}}\), \(F_{0}^{l,c}\), and \(F_{1}^{l,c}\) as input and yields upsampled optical flows \({\textbf{V}}_{t\rightarrow 0}^{l, \mathcal{P}_{\mathcal{S}}}\) and \({\textbf{V}}_{t\rightarrow 1}^{l, \mathcal{P}_{\mathcal{S}}}\) in a adaptive manner.
In the synthesis network (SN), \(I_{0}^{l}\), \(I_{1}^{l}\), \(F_{0}^{l,c}\), and \(F_{1}^{l,c}\) are backwarped~\cite{jaderberg2015spatial} by \({\textbf{V}}_{t\rightarrow 0}^{l, \mathcal{P}_{\mathcal{S}}}\) and \({\textbf{V}}_{t\rightarrow 1}^{l, \mathcal{P}_{\mathcal{S}}}\).
The warped frames and context features then finally yield a blending mask for two warped images, \(O^{l,\mathcal{P}_{\mathcal{S}}}\) and corresponding interpolated frame \(\hat{I}_t^{l,\mathcal{P}_{\mathcal{S}}}\).
Simple U-net~\cite{ronneberger2015u} architecture is employed as SN for our BiM-VFI.

The teacher process operates almost in the same manner as the student process except using ground truth target frame \(I_t^l\) as input pair with each of \(I_0^l\) and \(I_1^l\). Since \(I_t^l\) is used as part of input, the resulting \(\tilde{\textbf{V}}_{t\rightarrow0}^{l,\mathcal{P}_{\mathcal{T}}}\) and \(\tilde{\textbf{V}}_{t\rightarrow1}^{l,\mathcal{P}_{\mathcal{T}}}\) are more precise than \(\tilde{\textbf{V}}_{t\rightarrow0}^{l,\mathcal{P}_{\mathcal{S}}}\) and \(\tilde{\textbf{V}}_{t\rightarrow1}^{l,\mathcal{P}_{\mathcal{S}}}\) so they are used to supervise the learning of \(\tilde{\textbf{V}}_{t\rightarrow0}^{l,\mathcal{P}_{\mathcal{S}}}\) and \(\tilde{\textbf{V}}_{t\rightarrow1}^{l,\mathcal{P}_{\mathcal{S}}}\), as well as to compute the BiM for student process \(M_{t\rightarrow 0,1}^{l,\mathcal{P}_{\mathcal{S}}}\).

In the rest of this section, we will explain our motion description map BiM (\cref{subsec:motion field}), specific modules in our BiM-VFI (\cref{subsec:MFFN,subsec:caun}), and the proposed knowledge distillation strategy KDVCF (\cref{subsec:KDVCF}) in detail.
\subsection{Bidirectional Motion Field (BiM)}
\label{subsec:motion field}
% TODO 
Various non-uniform motions including accelerating, decelerating, and changing directions are contained in real-world videos.
Such non-uniform motions not only make ill-posedness at inference but also cause VFI learning to suffer from the TTL ambiguity at training time, thus resulting in severe blur artifacts in interpolated frames.
It is very challenging to resolve such ambiguity problems directly in inference time because it is difficult to predict the actual trajectory when only the first and last frames of the motion are given.
Instead, we take an alternative approach to solving the blur problem caused by TTL ambiguity at training time.

To resolve the TTL ambiguity during training, we introduce a Bidirectional Motion Field (BiM),  \(\textbf{M}_{t\rightarrow0,1}=[R,\Phi]^T\), as a novel motion descriptor that consists of pixel-wise motion magnitude ratios  \(R\) and angles \(\Phi\) between bidirectional optical flows \(\textbf{V}_{t\rightarrow 0}\) and \(\textbf{V}_{t\rightarrow 1}\) from \(I_t\) to each of \(I_0\) and \(I_1\).
The BiM at pixel location \((x,y)\) is defined as: 
\begin{equation}
\label{eq:BiM}
    \textbf{M}_{t\rightarrow0,1}(x,y) = [R(x,y), \Phi(x,y)]^T=[\frac{r_0}{r_0 + r_1}, \phi ]^T,
\end{equation}
where \(r_0 = ||\textbf{V}_{t\rightarrow0}(x,y)||\), \(r_1 = ||\textbf{V}_{t\rightarrow1}(x,y)||\), and \(\phi = \angle\textbf{V}_{t\rightarrow1}(x,y)-\angle\textbf{V}_{t\rightarrow0}(x,y)\) (top left of \cref{fig:motion_field}). 

\cref{fig:motion_field} depicts a TTL ambiguity comparison of different motion descriptors, time indexing~\cite{huang2022real,li2023amt,kong2022ifrnet,zhou2023exploring}, distance indexing~\cite{zhong2023clearer} and our BiM.
%First of all, it should be noted that the time indexing and distance indexing suffers from the TTL ambiguity while our BiM does not although \(I_t\) is used during training.
CASE 1 illustrates a car at timestep \(t=0.5\) is exactly at the center between the cars at timestep \(t=0\) and \(t=1\).
All motion descriptors can avoid TTL ambiguity if all training triplets have uniform motions as depicted in CASE1.
%In this case, the time indexing and distance indexing becomes equivalent, have no TTL ambiguity.
In CASE 2, the blue car is placed at a relative distance of 0.4, while the green car is placed at a relative distance of 0.6.
In this case, the blue and green cars are described by the same time indexing of 0.5 but different distance indexings and BiMs of 0.4 and 0.6, which incurs the TTL ambiguity only with time indexing.
That is, time indexing cannot distinguish the blue and green car's position, where both cars are captured at \(t = 0.5\).
Lastly, CASE 3 shows the case where the blue and green cars have different changes in motion directions at an accelerating speed.
Both time and distance indexings fail to distinguish the two cars, but our BiM can describe the two cars differently in terms of \(\phi = 1.2\pi\) and \(0.8\pi\).
Due to this TTL ambiguity problem, recent VFI models trained with time indexing~\cite{huang2022real,li2023amt,kong2022ifrnet,zhou2023exploring} or distance indexing~\cite{zhong2023clearer} tend to produce blurry interpolated frames at target times in the sense of averaging all possible answers (blue and green cars) to minimize the objectives such as L1 losses.

For inference, \(I_t\) has to be restored, but the flows \(\textbf{V}_{t\rightarrow 0}\), \(\textbf{V}_{t\rightarrow 1}\), or even the BiM \(\textbf{M}_{t\rightarrow0,1}\) is not available except for the target time \(t\).
Moreover, the motion types (uniform or non-uniform) are not known between \(I_0\) and \(I_1\).
Nevertheless, it is known that uniform motion assumption reasonably works well~\cite{zhong2023clearer}.
We extend this uniform assumption, \(\textbf{V}_{t\rightarrow 0} / t + \textbf{V}_{t\rightarrow 1} / (1-t) = 0\). by adding angle information \(\phi=\pi\).
So, the uniform BiM used for inference is given as: 
\begin{equation}
\label{eq:unifrom_bim}
    \textbf{M}_{t\rightarrow{0,1}}^\text{uni}=\begin{bmatrix}t \cdot \textbf{1}_{H\times W} & \pi \cdot \textbf{1}_{H\times W}
    \end{bmatrix}^T,
\end{equation}
where \(H\) and \(W\) are the height and width of desired bidirectional optical flows, \(\textbf{V}_{t\rightarrow0}\) and \(\textbf{V}_{t\rightarrow1}\).
We demonstrate that our BiM, under the uniform motion assumption, can interpolate frames with a similar sense of time indexing as described in \cref{subsec:uniform_motion_assumption}.

We point out that our BiM guides the BiM-VFI model to yield relatively cleaner interpolated frames at \(t\) than other VFI models with time indexing and distance indexing (\cref{fig:qual_arb}).
With the distinct motion describability of BiM, our BiM-VFI is trained without TTL ambiguity and can infer much cleaner target frames under the uniform motion assumption, although the motion of the interpolated target frames may not aligned with their real motions.

\subsection{BiM-guided FlowNet (BiMFN)}
\label{subsec:MFFN}
\begin{figure}[ht]
    \centering
    \includegraphics[width=0.9\columnwidth]{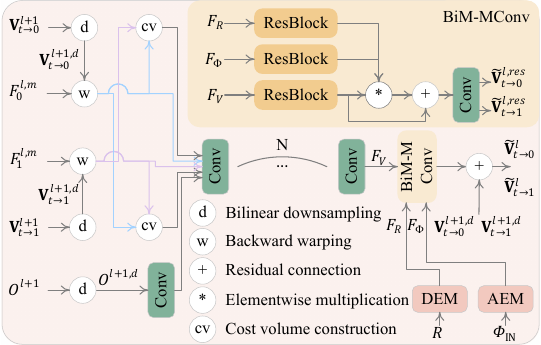}
    % \vspace{-0.2cm}
    \caption{{Proposed BiM-Guided FlowNet (BiMFN) at \textit{l}-th pyramid level.}}
    % \vspace{-0.3cm}
    \label{fig:mffn}
\end{figure}
We now design a bidirectional flow estimation network that utilizes the BiM. \cref{fig:mffn} shows our BiM-guided FlowNet (BiMFN) that estimates \(\tilde{\textbf{V}}_{t\rightarrow 0}^l\) and \(\tilde{\textbf{V}}_{t\rightarrow 1}^l\) from \(I_t^l\) to each of \(I_0^l\) and \(I_1^l\) at \(l\)-th pyramid level.
It is noted that \({\textbf{V}}_{t\rightarrow 0}^{l+1,d}\) and \({\textbf{V}}_{t\rightarrow 1}^{l+1,d}\) are bilinearly downsampled versions of \({\textbf{V}}_{t\rightarrow 0}^{l+1}\) and \({\textbf{V}}_{t\rightarrow 1}^{l+1}\) by a factor of 2 and its magnitude is also divided by factor of 2 to match the spatial sizes and magnitude of \(\tilde{\textbf{V}}^{l}_{t\rightarrow0}\) and \(\tilde{\textbf{V}}^{l}_{t\rightarrow1}\).
The blending mask estimated from the previous pyramid level, \(O^{l+1}\), is also downsampled as \(O^{l+1,d}\) in the same sense.
% The estimated flows, \(\textbf{V}_{t\rightarrow0}^{l-1,d}\) and \(\textbf{V}_{t\rightarrow1}^{l-1,d}\), and blending mask \(O^{l-1}\) from (\(l-1\))-th pyramid level and the extracted motion features, \(F^{l,m}_{0}\) and \(F^{l,m}_{1}\), at current \(l\)-th pyramid level are concatenated as input to the BiMFN. 
% Let 1/4-sized motion features of \(I^{l}_{0}\) and \(I^{l}_{1}\) as \(F^{l}_{0}\) and \(F^{l}_{1}\), respectively. Initially, BiMFN adjusts the scale and magnitudes of given inputs. Estimated flows in last pyramid levels, \(V_{t\rightarrow0}^{l-1}\) and \(V_{t\rightarrow1}^{l-1}\), have dimensions reduced by a factor of 2 in both height and width compared to \(I^{l}_{0}\) and \(I^{l}_{1}\), whereas \(F^{l}_{0}\) and \(F^{l}_{1}\) are downsampled one-quarter the height and width of \(I^{l}_{0}\) and \(I^{l}_{1}\). Hence, we bilinearly downsample \(V_{t\rightarrow0}^{l-1}\) and \(V_{t\rightarrow1}^{l-1}\) by a scale factor of 2 and divide the magnitude of these flows by 2, to match the dimensions of \(F^{l}_{0}\) and \(F^{l}_{1}\). BiMFN also utilize the blending masks from previous level, \(O^{l-1}\), where blending masks contain the information about occlusion in given sequences. Blending masks are also downsampled to match the size of other inputs.
% Since only the downsampled version of \(O^{l-1}\), \(V_{t\rightarrow0}^{l-1}\) and \(V_{t\rightarrow1}^{l-1}\) are used in the rest of this section, we omitted any subscript or superscript to represent downsampled version of \(O^{l-1}\), \(V_{t\rightarrow0}^{l-1}\) and \(V_{t\rightarrow1}^{l-1}\) for simplicity. 

The BiMFN first warps \(F^{l,m}_{0}\) and \(F^{l,m}_{1}\) using \(\textbf{V}_{t\rightarrow0}^{l+1}\) and \(\textbf{V}_{t\rightarrow1}^{l+1}\), resulting in \(F^{l,m}_{0\rightarrow t}\) and \(F^{l,m}_{1 \rightarrow t}\).
\(F^{l,m}_{0\rightarrow t}\) and \(F^{l,m}_{1 \rightarrow t}\) are then used to construct local cost volumes~\cite{sun2018pwc, teed2020raft} to find precise correspondences between two features.
Local cost volumes from \(F^{l,m}_{0\rightarrow t}\) to \(F^{l,m}_{1 \rightarrow t}\) and vice versa are constructed to encapsulate asymmetric correspondence information.
\(O^{l+1,d}\) is encoded through separate convolution layers and then concatenated with \(F^{l,m}_{0\rightarrow t}\), \(F^{l,m}_{1 \rightarrow t}\), and two cost volumes before being passed to the next module.

Then, the BiM Modulation Convolution (BiM-MConv) in \cref{fig:mffn} takes three inputs with \(F_V\), \(F_R\), and \(F_\Phi\) where (i) \(F_V\) is the output of the cascaded eight convolution layers; (ii) \(F_R\) is the encoded output from the Distance Embedding Module (DEM) with a one-channel motion ratio component input \(R\) in \cref{eq:BiM}; and (iii) \(F_\Phi\) is the feature output of the Angle Embedding Module (AEM) for a two-channel angular component input \(\Phi_\text{IN}=(\sin(\Phi), \cos(\Phi))\) from \cref{eq:BiM}.
Finally, for \(F_R\), \(F_\Phi\), and \(F_V\) input, the BiM-MConv integrates them with elementwise multiplication and produces the refined residual flow fields, \(\tilde{\textbf{V}}_{t\rightarrow0}^{l,res}\) and \(\tilde{\textbf{V}}_{t\rightarrow1}^{l,res}\).
The final flow estimations at \(l\)-th pyramid level are computed as:
\begin{equation}
    \begin{split}
        &\tilde{\textbf{V}}_{t\rightarrow0}^{l} = \textbf{V}_{t\rightarrow0}^{l+1,d} + \tilde{\textbf{V}}_{t\rightarrow0}^{l,res},\\
        &\tilde{\textbf{V}}_{t\rightarrow1}^{l} = \textbf{V}_{t\rightarrow1}^{l+1,d} + \tilde{\textbf{V}}_{t\rightarrow1}^{l,res}.
    \end{split}
\end{equation}

\subsection{Content-Aware Upsampling Network (CAUN)}
\label{subsec:caun}
Since \(\tilde{\textbf{V}}_{t\rightarrow0}^{l}\) and \(\tilde{\textbf{V}}_{t\rightarrow1}^{l}\) are of the same size as \(F^{l}_{0,m}\) and \(F^{l}_{1,m}\), they must be upsampled by a scale factor of 4 to match the image dimensions, \(H/2^l\) and \(W/2^l\).
In general, the usage of bilinear or bicubic upsampling can incur flow leakages along object boundaries and diminish small objects by blending external flows~\cite{teed2020raft, luo2021upflow}.
To avoid this, adaptive upsampling is commonly employed in optical flow estimation models such as~\cite{hui2018liteflownet, teed2020raft, luo2021upflow}, but still is not widely used in VFI models.
We adopt `Convex upsample' layer, proposed by Teed \etal~\cite{teed2020raft}, and extend it for VFI to upsample \(\tilde{\textbf{V}}_{t\rightarrow0}^{l}\) and \(\tilde{\textbf{V}}_{t\rightarrow1}^{l}\)  to \(\textbf{V}_{t\rightarrow0}^{l}\) and \(\textbf{V}_{t\rightarrow1}^{l}\) using pixel-wise adaptive kernels.
% The extension is as follows: since the context feature input, \(F^{l,c}_{0}\) and \(F^{l,c}_{1}\), are not anchored to target time \(t\), they are first warped to \(t\) by \(\tilde{\textbf{V}}_{t\rightarrow0}^{l}\) and \(\tilde{\textbf{V}}_{t\rightarrow1}^{l}\) as:
% \begin{equation}
%     \begin{split}
%         &F^{l,c}_{0\rightarrow t} = \text{Warp}(F^{l,c}_{0}, \tilde{\textbf{V}}_{t\rightarrow0}^{l}), \\
%         &F^{l,c}_{1\rightarrow t} = \text{Warp}(F^{l,c}_{1}, \tilde{\textbf{V}}_{t\rightarrow1}^{l}),
%     \end{split}
% \end{equation}
% where \(\text{Warp}(\cdot, \cdot)\) is a backward warping function~\cite{jaderberg2015spatial}. Then our content-aware upsampled flows are obtained as: 
% \begin{equation}
%     \begin{split}
%         [\textbf{V}_{t\rightarrow 0}^l, \textbf{V}_{t\rightarrow 1}^l] = \text{CAUN}(\tilde{\textbf{V}}_{t\rightarrow 0}^l,\tilde{\textbf{V}}_{t\rightarrow 1}^l,F^{l,c,w}_{0\rightarrow t},F^{l,c,w}_{1\rightarrow t})
%     \end{split}
% \end{equation}
% where CAUN is our context-aware upsampling network shown in \cref{fig:main}.
The detailed structure of CAUN is presented in \textit{Supplementals}.
The adaptive upsampling of the CAUN not only aesthetically enhances the upsampling of flows but it also more effectively captures small objects and complex boundaries in interpolated frames, thanks to its ability to maintain the flows of small objects and prevent mixing the flows at object boundaries.
Our extensive experiments show its effectiveness, which is presented in \cref{subsec:ablation}.

\begin{figure*}[ht]
\centering
    \includegraphics[width=\textwidth]{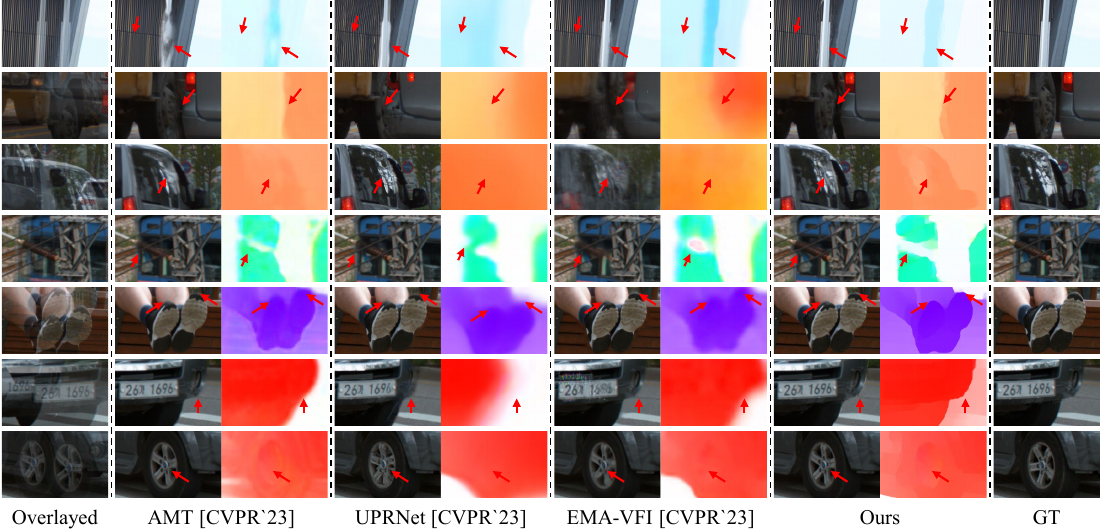}
    % \vspace{-0.6cm}
    \caption{Qualitative comparison for fixed-time interpolation datasets, XTest~\cite{sim2021xvfi}.}
    % \vspace{-0.3cm}
    \label{fig:qual_fix}
\end{figure*}
\subsection{Knowledge Distillation for VFI-centric Flow Supervision (KDVCF)}
\label{subsec:KDVCF}
As discussed in \cref{subsec:vfi}, the flow estimation of pre-trained optical flow models and VFI models operates differently in certain areas, such as blurs and shadows.
Therefore, instead of using pre-trained optical flow models for BiM and flow supervision, we propose Knowledge Distillation for VFI-centric Flow supervision (KDVCF) that provides BiM and flow supervision more suitable for VFI.
KDVCF consists of student process \(\mathcal{P}_{\mathcal{S}}\) and teacher process \(\mathcal{P}_{\mathcal{T}}\).
The two processes sequentially run but share the same architecture with the same weights.
First, as shown in~\cref{fig:main}, \(\mathcal{P}_{\mathcal{T}}\) takes two input pairs, (\(I_0^l, I_t^l\)) and (\(I_t^l,I_1^l\)), and produce precise flow estimations \(\tilde{\textbf{V}}_{t\rightarrow0}^{l,\mathcal{P}_{\mathcal{T}}}\) and \(\tilde{\textbf{V}}_{t\rightarrow1}^{l,\mathcal{P}_{\mathcal{T}}}\). Then the BiM is computed based on \(\tilde{\textbf{V}}_{t\rightarrow0}^{l,\mathcal{P}_{\mathcal{T}}}\) and \(\tilde{\textbf{V}}_{t\rightarrow1}^{l,\mathcal{P}_{\mathcal{T}}}\) according to~\cref{eq:BiM}, and is inputted to the BiMFN of \(\mathcal{P}_{\mathcal{S}}\).
So, the knowledge distillation can be made from \(\mathcal{P}_{\mathcal{T}}\) to \(\mathcal{P}_{\mathcal{S}}\) for flow estimations during training.
Note that \(\mathcal{P}_{\mathcal{S}}\) only remains at inference.

In \(\mathcal{P}_{\mathcal{T}}\), the BiMFN yields \(\tilde{\textbf{V}}_{t\rightarrow0}^{l,\mathcal{P}_{\mathcal{T}}}\) and \(\tilde{\textbf{V}}_{t\rightarrow t|0t}^{l,\mathcal{P}_{\mathcal{T}}}\) for input frame pair \(I_0^l\) and \(I_t^l\).
In this case, \(\tilde{\textbf{V}}_{t\rightarrow 0}^{\mathcal{P}_{\mathcal{T}}}\) can be an accurate flow field estimate, and \(\tilde{\textbf{V}}_{t\rightarrow t|0t}^{l,\mathcal{P}_{\mathcal{T}}}\) is ideally a vector field of all \(\textbf{0}\)'s.
Since our BiM is formatted in terms of a relative motion ratio and an angle between bidirectional flows, the resulting motion ratio \(R\) is constructed as a uniform map of 1's but the angles map \(\Phi\) is undefined because \(\tilde{\textbf{V}}_{t\rightarrow t}^{l,\mathcal{P}_{\mathcal{T}}}\) has zero vectors.
So, \(\Phi\) is filled with angles randomly sampled from a uniform distribution \(\mathcal{U}(0, 2\pi)\) to avoid a bias to any specific angle value.
Thus, the BiM to be inputted to the BiMFN in \(\mathcal{P}_{\mathcal{T}}\) for input pair (\(I_0^l\), \(I_t^l\)) is defined as:
\begin{equation}
\label{eq:teacher_bim0}   
 \textbf{M}_{t\rightarrow{0,t}}^{l,\mathcal{P}_{\mathcal{T}}}=\begin{bmatrix}\textbf{1}_{H/2^{l+2}\times W/2^{l+2}}& \phi_0 \cdot \textbf{1}_{H/2^{l+2}\times W/2^{l+2}}
    \end{bmatrix}^T
\end{equation}
where \(\phi_0 \sim \mathcal{U}(0, 2\pi)\).
% The BiM in \cref{eq:teacher_bim} is reasonable because it can have precise motion \(\tilde{\textbf{V}}_{t\rightarrow 0}^{l,\mathcal{P}_{\mathcal{T}}}\) from time $t$ to time $0$, but its random angles can be ignored via the BiMFN which can yield accurate \(\tilde{\textbf{V}}_{t\rightarrow 0}^{l,\mathcal{P}_{\mathcal{T}}}\).
% This \(\tilde{\textbf{V}}_{t\rightarrow 0}^{l,\mathcal{P}_{\mathcal{T}}}\) can then be used for the BiMFN in \(\mathcal{P}_{\mathcal{S}}\) as supervision signals for \(\tilde{\textbf{V}}_{t\rightarrow 0}^{l,\mathcal{P}_{\mathcal{S}}}\) during training.
Like for another input pair (\(I_t^l\), \(I_1^l\)) in \(\mathcal{P}_{\mathcal{T}}\), the BiMFN yields \(\tilde{\textbf{V}}_{t\rightarrow t|t1}^{l,\mathcal{P}_{\mathcal{T}}}\) and \(\tilde{\textbf{V}}_{t\rightarrow 1}^{l,\mathcal{P}_{\mathcal{T}}}\).
So, the BiM as input to the BiMFN for input pair (\(I_t^l\),\(I_1^l\)) is defined as:
\begin{equation}
\label{eq:teacher_bim1}   
    \textbf{M}_{t\rightarrow{t,1}}^{l,\mathcal{P}_{\mathcal{T}}}=\begin{bmatrix}\textbf{0}_{H/2^{l+2}\times W/2^{l+2}}& \phi_1 \cdot \textbf{1}_{H/2^{l+2}\times W/2^{l+2}}
    \end{bmatrix}^T
\end{equation}
where \(\phi_1 \sim \mathcal{U}(0, 2\pi)\).
% In each \(l\)-th pyramid level of \(\mathcal{P}_{\mathcal{T}}\), the estimated flows \(V_{t\rightarrow0}^{l-1, \mathcal{P}_{\mathcal{S}}}\) and \(V_{t\rightarrow1}^{l-1, \mathcal{P}_{\mathcal{S}}}\) in the previous pyramid level of \(\mathcal{P}_{\mathcal{S}}\), downsampled source images, \(I_0^l\) and \(I_1^l\), and downsampled target image \(I_t^l\) are available. Each of \(I_0^l\), \(I_1^l\) and \(I_t^l\) is passed through the MFE to extract the motion features \(F_i^{l,m}\), \((i=0, t, 1)\). \(F_0^{l,m}\) and \(F_t^{l,m}\) are passed to the BiMFN with BiM of , and \(F_t^{l,m}\) and \(F_1^{l,m}\) are passed to the BiMFN with BiM of 

% Flow fields \(\tilde{\textbf{V}}_{t\rightarrow 0}^{l,\mathcal{P}_{\mathcal{T}}}\) and \(\tilde{\textbf{V}}_{t\rightarrow 1}^{l,\mathcal{P}_{\mathcal{T}}}\) can be accurately estimated in \(\mathcal{P}_{\mathcal{T}}\) due to the availability of \(I_t^l\).
To ensure the VFI-centric estimation of \(\tilde{\textbf{V}}_{t\rightarrow 0}^{l,\mathcal{P}_{\mathcal{T}}}\) and \(\tilde{\textbf{V}}_{t\rightarrow 1}^{l,\mathcal{P}_{\mathcal{T}}}\), these flows are further employed to reconstruct the target image \(\hat{I}_t^{l,\mathcal{P}_{\mathcal{T}}}\) in \(\mathcal{P}_{\mathcal{T}}\), along with CAUN and SN, and are trained with a photometric reconstruction loss.
Note that \(\tilde{\textbf{V}}_{t\rightarrow 0}^{l,\mathcal{P}_{\mathcal{T}}}\) and \(\tilde{\textbf{V}}_{t\rightarrow 1}^{l,\mathcal{P}_{\mathcal{T}}}\) are also employed to construct the BiM in \cref{eq:BiM} for \(\mathcal{P}_{\mathcal{S}}\) that operates for source frame pairs, \(I_0^l\) and \(I_1^l\), as well as to supervise the output flow fields, \(\tilde{\textbf{V}}_{t\rightarrow 0}^{l,\mathcal{P}_{\mathcal{S}}}\) and \(\tilde{\textbf{V}}_{t\rightarrow 1}^{l,\mathcal{P}_{\mathcal{S}}}\), of the BiMFN in \(\mathcal{P}_{\mathcal{S}}\).
Unlike estimated flows from pre-trained supervised optical flow models, that are trained to minimize end-point error with GT flows, our \(\tilde{\textbf{V}}_{t\rightarrow 0}^{l,\mathcal{P}_{\mathcal{T}}}\) and \(\tilde{\textbf{V}}_{t\rightarrow 1}^{l,\mathcal{P}_{\mathcal{T}}}\) align precisely with the objectives of VFI. Consequently, the distillation is fully beneficial and effectively tailored to the \(\mathcal{P}_{\mathcal{S}}\)'s purpose.
Extensive experiments showed that our KDVCF is more beneficial than the distillation from pre-trained supervised optical flow models.

During \(\mathcal{P}_{\mathcal{S}}\), our BiM-VFI learns various uniform and non-uniform motions with a distinct motion descriptor (BiM) and a precise VFI-centric flow supervision produced by \(\mathcal{P}_{\mathcal{T}}\).
It is worth menting that \(\tilde{\textbf{V}}_{t\rightarrow 0}^{l,\mathcal{P}_{\mathcal{S}}}\) and \(\tilde{\textbf{V}}_{t\rightarrow 1}^{l,\mathcal{P}_{\mathcal{S}}}\) in \(\mathcal{P}_{\mathcal{S}}\) can be precisely learned in the help of our BiM based on accurate flow fields \(\tilde{\textbf{V}}_{t\rightarrow 0}^{l,\mathcal{P}_{\mathcal{T}}}\) and \(\tilde{\textbf{V}}_{t\rightarrow 1}^{l,\mathcal{P}_{\mathcal{T}}}\) obtained in \(\mathcal{P}_{\mathcal{T}}\), with the availability of target frame \(I_t^l\).
For inference, the BiM for uniform motions is fed into the BiMFN, and our BiM-VFI can correspondingly construct clean interpolated frames with uniform motions although they might not be well aligned with ground truth target frames with non-uniform motions.

%% file: sec/4_results.tex
\section{Experiments}

\begin{table*}
\centering
{\small
\setlength\tabcolsep{3pt}
\begin{tabular}{l | c c c c c | c c c | c c c | c c c}
\multirow{3}{*}{Methods} & \multicolumn{5}{c}{\multirow{2}{*}{Vimeo90K-septuplet}} & \multicolumn{9}{c}{SNU-FILM-arb} \\
 &  &  &  &  &  & \multicolumn{3}{c}{Medium} & \multicolumn{3}{c}{Hard} & \multicolumn{3}{c}{Extreme} \\
 & psnr & ssim & lpips & stlpips & niqe & lpips & stlpips & niqe & lpips & stlpips & niqe & lpips & stlpips & niqe \\
 \hline\hline
RIFE~\cite{huang2022real} & 28.22 & 0.912 & 0.105 & 0.084 & 6.663 & 0.038 & 0.021 & 4.975 & 0.072 & 0.054 & 5.177 & 0.134 & 0.116 & 5.358 \\
IFRNet~\cite{kong2022ifrnet} & 28.26 & 0.915 & 0.088 & 0.094 & 6.422 & 0.046 & 0.037 & 4.921 & 0.066 & 0.054 & 4.870 & 0.115 & 0.094 & 4.793 \\
M2M-PWC~\cite{hu2022many} & 27.43 & 0.901 & 0.081 & 0.055 & \underline{6.026} & 0.030 & 0.014 & 4.806 & 0.049 & 0.025 & \underline{4.758} & \underline{0.089} & \underline{0.055} & \textbf{4.657} \\
AMT-S~\cite{li2023amt} & \underline{28.52} & \underline{0.920} & 0.101 & 0.105 & 6.866 & 0.072 & 0.046 & 5.443 & 0.089 & 0.060 & 5.444 & 0.135 & 0.098 & 5.500 \\
UPRNet~\cite{jin2023unified} & 27.23 & 0.900 & 0.087 & 0.061 & 6.280 & 0.031 & 0.014 & 4.837 & 0.054 & 0.028 & 4.909 & 0.092 & 0.056 & 4.923 \\
EMA-VFI~\cite{zhou2023exploring} & \textbf{29.41} & \textbf{0.928} & 0.086 & 0.079 & 6.736 & 0.041 & 0.025 & 4.984 & 0.072 & 0.054 & 5.236 & 0.125 & 0.106 & 5.522 \\
\hline
{[D,R]}-RIFE~\cite{zhong2023clearer} & 27.41 & 0.901 & 0.086 & 0.059 & 6.220 & 0.027 & 0.011 & 4.751 & 0.050 & 0.026 & 4.829 & 0.101 & 0.072 & 4.898 \\
{[D,R]}-IFRNet~\cite{zhong2023clearer} & 27.13 & 0.899 & \underline{0.078} & 0.053 & {6.167} & \underline{0.026} & \underline{0.010} & 4.757 & \underline{0.048} & \underline{0.023} & {4.798} & {0.095} & {0.062} & 4.821 \\
{[D,R]}-AMT-S~\cite{zhong2023clearer} & 27.17 & 0.902 & 0.081 & 0.053 & 6.326 & 0.029 & 0.013 & \underline{4.747} & 0.054 & 0.028 & 4.849 & 0.107 & 0.071 & 5.017 \\
{[D,R]}-EMA-VFI~\cite{zhong2023clearer} & 24.73 & 0.851 & 0.081 & \underline{0.046}& 6.227 & 0.032 & 0.013 & 4.864 & 0.055 & 0.027 & 4.978 & 0.106 & 0.071 & 5.120 \\
\hline
Ours & 26.83 & 0.893 & \textbf{0.070} & \textbf{0.043} &\textbf{ 6.009} & \textbf{0.023} & \textbf{0.008} & \textbf{4.693} & \textbf{0.039} & \textbf{0.015} & \textbf{4.725} & \textbf{0.070} & \textbf{0.034} & \underline{4.751} \\

\end{tabular}
}
% \vspace{-0.2cm}
\caption{Quantitative comparisons on arbitrary-time interpolation datasets.}
% \vspace{-0.3cm}
\label{tab:quan_arb}
\end{table*}
\subsection{Experiments details}
\textbf{Benchmarks.}
We tested our BiM-VFI for both fixed-time (\(t\)=0.5) and arbitrary time interpolation datasets.
For fixed-time interpolation, we used Vimeo90K triplet~\cite{xue2019video}, SNU-FILM~\cite{choi2020channel}, and XTest~\cite{sim2021xvfi} datasets.
For arbitrary-time interpolation, we conducted experiments on Vimeo90K septuplet~\cite{xue2019video} and SNU-FILM-arb~\cite{guo2024generalizable} datasets.

\noindent\textbf{Metrics.}
We measured both pixel-centric metrics (PSNR and SSIM) and perceptual metrics (LPIPS~\cite{zhang2018unreasonable}, STLPIPS~\cite{ghildyal2022shift}, and NIQE~\cite{mittal2012making}) for quantitative comparisons between our BiM-VFI and SOTA methods.
Both LPIPS and STLPIPS are full-reference perceptual metrics, while NIQE is a no-reference perceptual metric.
LPIPS and STLPIPS compute the similarity between features of the input and reference images using a pre-trained network.
However, while LPIPS exhibits a significant drop in metric performance in the presence of minor misalignments, STLPIPS is more tolerant of such misalignments.

\subsection{Qualitative results}
We compared our BiM-VFI with both State-of-the-art (SOTA) arbitrary-time and fixed-time VFI models for various datasets.
\cref{fig:qual_arb} compares arbitrary-time interpolation results at \(t=0.25\), \(0.5\), and \(0.75\) for SNU-FILM-arb extreme datasets~\cite{choi2020channel}.
While the objects (dog's heads and balls in the upper figures, legs in the lower figures) with very fast motions are blurrily reconstructed by all the SOTA models, including the models plugged with `distance indexing' and `iterative reference-based estimation'~\cite{zhong2023clearer} (denoted as [D,R]), our BiM-VFI successfully restored much cleaner frames than other methods.

Also, we compared our BiM-VFI with other SOTA models (\cref{fig:qual_fix}) for a fixed-time VFI dataset, XTest-set~\cite{sim2021xvfi}.
Even though other SOTA VFI models are trained on fixed-time datasets, our BiM-VFI outperforms them qualitatively.
As shown in \cref{fig:qual_fix}, the small objects (streetlight pole in the 1st row, power lines in the 3rd row, car plate numbers in the 5th row) and the complex boundaries between a car wheel and a rear bumper in the 2nd row are well constructed in the interpolated frame by our BiM-VFI, while the other SOTA models fail to interpolate repeated patterns (building wall with vertical strips in the 1st row) or incurred blurs in object boundaries.
It is worth noting for the estimated flows in the 3rd, 5th, 7th, and 9th columns that our BiM-VFI can estimate sharper flows even in object boundaries and repeated patterns compared to other SOTA models.

\begin{table*}
\centering
{\small
\setlength\tabcolsep{3pt}
\begin{tabular}{l | c | c | c | c | c | c}
\multirow{3}{*}{Methods} & \multirow{2}{*}{Vimeo 90K-triplet} & \multicolumn{4}{c}{SNU-FILM} & \multirow{2}{*}{XTest} \\
 &  & Easy & Medium & Hard & Extreme &   \\
 \hline\hline
AMT-G~\cite{li2023amt}  & \textbf{0.019}/\textbf{0.012}/{5.327} & 0.022/{0.008}/4.822 & 0.035/{0.015}/4.924 & 0.060/0.028/4.993 & 0.112/0.068/4.993  & 0.134/0.097/6.883 \\
M2M-PWC~\cite{hu2022many}  & 0.025/0.017/5.346 & 0.021/0.009/4.765 & 0.036/0.016/{4.824} & 0.063/0.033/\textbf{4.773}  & 0.212/\underline{0.057}/6.082 & 0.211/0.135/6.005 \\
UPRNet~\cite{jin2023unified}  & 0.022/0.015/5.367 & \underline{0.018}/{0.008}/\underline{4.703} & 0.034/{0.015}/4.853 & 0.062/0.030/4.975 & \underline{0.110}/0.067/{5.052}  & \underline{0.095}/\textbf{0.059}/{6.372} \\
RIFE~\cite{huang2022real}  & 0.022/0.014/\underline{5.240} & \underline{0.018}/\underline{0.007}/4.709 & \underline{0.032}/0.014/\underline{4.813} & 0.066/0.037/4.872 & 0.138/0.099/4.935 & 0.295/0.209/6.419 \\
XVFI~\cite{sim2021xvfi}  & 0.028/0.019/\textbf{5.236} & 0.027/0.015/4.829 & 0.040/0.024/4.847 & 0.068/0.043/\underline{4.780} & 0.120/0.083/\textbf{4.618}  & 0.109/0.072/\underline{6.041} \\
IFRNet~\cite{kong2022ifrnet}  & \textbf{0.019}/\underline{0.013}/5.267 & 0.020/0.008/4.820 & \underline{0.032}/\underline{0.013}/4.889 & \underline{0.056}/\underline{0.027}/4.890 & 0.113/0.073/\underline{4.856}  & 0.190/0.134/\textbf{5.892} \\
EMA-VFI~\cite{zhou2023exploring} & \underline{0.020}/\underline{0.013}/5.350 & 0.019/{0.008}/4.704 & {0.033}/{0.015}/4.847 & {0.059}/0.030/4.979 & 0.113/0.073/5.087 & 0.139/0.099/7.008 \\
Ours & \underline{0.020}/\textbf{0.012}/{5.283} & \textbf{0.017}/\textbf{0.006}/\textbf{4.678} & \textbf{0.029}/\textbf{0.011}/\textbf{4.773} & \textbf{0.052}/\textbf{0.022}/{4.863} & \textbf{0.097}/\textbf{0.052}/{4.942}  & \textbf{0.089}/\underline{0.060}/{6.717} \\

\end{tabular}
}
% \vspace{-0.3cm}
\caption{Quantitative comparisons on fixed time interpolation datasets.}
% \vspace{-0.3cm}
\label{tab:quan_fix}
\end{table*}

% \begin{table}
% \centering
% {\small
% \setlength\tabcolsep{3pt}
% \begin{tabular}{ l | c c c c c }
% Method & PSNR & SSIM & LPIPS & STLPIPS & NIQE \\
% \hline
% Ours (R, \(\Phi\)) & 27.63 & 0.905 & \textbf{0.076} & \textbf{0.049} & \textbf{6.334} \\
% (T) & \textbf{28.73} & \textbf{0.921} & 0.094 & 0.074 & 6.802 \\
% (R) & \multicolumn{5}{c}{\textbf{Train Failed}} \\
% (T, \(\Phi\)) & 27.63 & 0.905 & 0.076 & 0.049 & 6.331 \\

% \end{tabular}
% }
% \caption{Ablation studies on Motion Field.}
% \label{tab:ablation_mf}
% \end{table}

% \begin{table}
% \centering
% {\small
% \setlength\tabcolsep{3pt}
% \begin{tabular}{ l | c c c c c }
% Method & PSNR & SSIM & LPIPS & STLPIPS & NIQE \\
% \hline
% Ours & 26.83 & 0.8929 & \textbf{0.07038} & \textbf{0.04258} & \textbf{6.009} \\
% w/o Flow loss & 27.56 & 0.904 & 0.076 & 0.05 & 6.358 \\
% Flow loss & \textbf{27.63} & \textbf{0.905} & 0.076 & 0.049 & 6.334 \\

% \end{tabular}
% }
% \caption{Ablation studies on KDVCF.}
% \label{tab:ablation_training}
% \end{table}

% \begin{table}
% \centering
% {\small
% \setlength\tabcolsep{3pt}
% \begin{tabular}{ l | c c c c c }
% Method & PSNR & SSIM & LPIPS & STLPIPS & NIQE \\
% \hline
% Ours (MFFN) & 26.83 & 0.893 & \textbf{0.070} & \textbf{0.043} & \textbf{6.009}  \\
% MF concat & 26.79 & 0.893 & 0.071 & 0.044 & 6.059 \\
% w/o GUN &  &  &  &  &  \\

% \end{tabular}
% }
% \caption{Ablation studies on model architecture.}
% \label{tab:ablation_model}
% \end{table}

\begin{table}[ht]
\centering
{\small
\setlength\tabcolsep{3pt}

\begin{tabular}{l | l | ccc}
Ablation & Methods & lpips & stlpips & niqe \\
\hline
\hline
\multirow{3}{*}{(a) BiM} & \((T)\) & 0.098 & 0.077 & 6.838 \\
 & \((R)\) & \multicolumn{3}{c}{Train failed} \\
 & \((T,\Phi)\) & 0.074 & 0.045 & 6.222 \\
 \hline
\multirow{2}{*}{(b) KDVCF} & w/o KDVCF & 0.076 & 0.050 & 6.358 \\
 & w\ Flow loss & 0.076 & 0.049 & 6.334 \\
 \hline
\multirow{2}{*}{(c) Modules} & BiM concat & 0.071 & 0.044 & 6.059 \\
 & w/o CAUN & 0.076 & 0.045 & 6.124 \\
\hline
Full & Ours & \textbf{0.070} & \textbf{0.043} & \textbf{6.009} \\

\end{tabular}
}
% \vspace{-0.3cm}
\caption{Ablation studies on BiM, KDVCF, and modules.}
% \vspace{-0.3cm}
\label{tab:ablation_all}
\end{table}

\subsection{Quantitative results}
\label{subsec:quan}
\cref{tab:quan_arb} compares our BiM-VFI and SOTA methods for arbitrary time interpolation test datasets.
While our BiM-VFI underperforms in terms of pixel-centric metrics for Vimeo90K-septuplet~\cite{xue2019video}, it demonstrated significantly higher performance by a large margin in terms of the perceptual metrics for Vimeo90K-septuplet and SNU-FILM-arb~\cite{guo2024generalizable}.
For the Vimeo90K-septuplet and the SNU-FILM-arb (Medium, Hard and Extreme) data sets, our BiM-VFI outperformed all other VFI models in terms of perceptual metrics (LPIPS, STLPIPS, and NIQE) except M2M-PWC~\cite{hu2022many} with 4.657 only in NIQE metric for SNU-FILM-arb Extreme data set.
As shown in \cref{tab:quan_arb}, there are large metric gaps in pixel-centric metrics (PSNR and SSIM) between our BiM-VFI and most of the other SOTA methods because our BiM-VFI assumes uniform motion in interpolated frame reconstruction at inference.
In spite of relatively large values of pixel-centric metrics for the other SOTA methods, their reconstructed interpolated frames are very blurry as shown in \cref{fig:qual_arb}.
These pixel-centric metrics conducted on test datasets containing non-uniform motions do not match the perceptual qualities as reported in \cite{zhong2023clearer, kiefhaber2024benchmarking}.
% We further conducted a user study to ensure our interpolated videos under uniform motion assumption are perceptually plausible.
% The detailed descriptions of the user study are reported in \textit{Supplementals}.

\cref{tab:quan_fix} shows the frame interpolation results for fixed-time interpolation data sets.
As shown, our BiM-VFI also achieved comparable or superior performance to other SOTA models across most datasets although it was not trained for frame interpolation tasks at fixed target times.
% \jh{We argue that in most applications, the goal of VFI is not to predict pixel-wise aligned frames, but to generate plausible frames with high perceptual quality. Furthermore, pixel-centric metrics are less sensitive to blur [51], the major artifact introduced by velocity ambiguity. The pixel-centric metrics are thus less
% informative and denoted in gray. On perceptual metrics (especially NIQE), the
% enhanced models significantly outperforms the base model. This consistency with
% our qualitative observations further validates the effectiveness of distance indexing and iterative reference-based estimation. [from "InterpAny-Clearer"]}

\subsection{Ablation studies}
\label{subsec:ablation}
We conducted ablation studies on the proposed BiM, KDVCF, and model components.
First, for the BiM, we performed ablations by replacing \(R\) with time indexing \(T\) and by removing the \(\Phi\) component, where we supervised the experiments using our proposed KDVCF on Vimeo90K septuplet datasets~\cite{xue2019video}.
When the \(\Phi\) component was removed, we excluded corresponding network components from the BiMFN.
As shown in \cref{tab:ablation_all} (a), our proposed BiM achieved the best perceptual metric, which indicates BiM can resolve ambiguity at training, thus interpolating frames with much fewer blurs under uniform-motion scenarios.

In \cref{tab:ablation_all} (b), we compared our proposed KDVCF with the supervision using FlowFormer~\cite{huang2022flowformer}-extracted flows and the training without any flow supervision.
As shown, our proposed KDVCF yielded higher perceptual metrics than the pre-trained flow supervision, confirming that our flow estimations from \(\mathcal{P}_{\mathcal{T}}\) provide more suitable BiM and supervision for training our BiM-VFI.

% \begin{figure}[ht]
% \centering
%     \includegraphics[width=0.8\columnwidth]{figures/ablation.pdf}
%     \caption{Visual examples of generating non-uniform motions. Each row is the result of VFI by BiM-VFI, with BiMs of \((t, 0.6\pi)\), \((t, \pi)\), and \((t, 1.2\pi)\), respectively.}
%     \label{fig:ablation}
% \end{figure}
Lastly, in \cref{tab:ablation_all} (c), we compared replacing elementwise multiplication in BiM-MConv with a module that concatenates \(F_V\), \(F_R\), and \(F_{\Phi}\) followed by a convolution layer, and removing the adaptive upsampling (CAUN).
Our BiMFN design was found to better leverage the BiM, and the CAUN not only effectively upsamples the flow but also improves the perceptual quality of frame interpolation results.
% \cref{fig:ablation} shows that the flow estimation with CAUN has sharper boundaries and more adequte estimations than the flows without CAUN.

\subsection{Limitation on Uniform Motion Assumption}
\label{subsec:uniform_motion_assumption}
Our BiM-VFI is limited for frame interpolation under a uniform motion assumption at inference time, due to the unavailability of the target BiM, which is an inherit limitation as for all other VFI methods. \cref{fig:limitation} demonstrates that other VFI methods, such as those employing time indexing (EMA-VFI~\cite{zhou2023exploring}) and distance indexing ([D,R]-EMA-VFI~\cite{zhong2023clearer}), also fail to adequately interpolate the target frame with non-uniform motions, where the boundary of the tree at the target frame (indicated by the blue line) does not align with the interpolation results from EMA-VFI, [D,R]-EMA-VFI, or our BiM-VFI (indicated by the green line).
Moreover, the boundary of the tree from the interpolated frame using EMA-VFI that employs time indexing is well aligned with those of models using distance indexing or BiM under the uniform motion assumption.
This suggests that the time-indexing-based method implicitly tends to comply with the uniform motion assumption from training where uniform motion is dominant, thereby supporting that the uniform motion assumption at inference is more likely enforced for all the methods in VFI.
\begin{figure}[ht]
\centering
    \includegraphics[width=0.8\columnwidth]{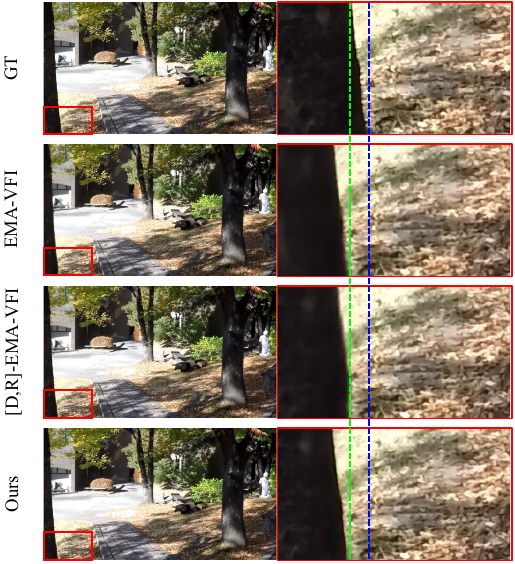}
    \caption{Visual comparisons of interpolating non-uniform motion target frame. 2nd column is zoom-in version of the red boxes in 1st column.}
    \label{fig:limitation}
\end{figure}

% \subsection{application}
% \label{subsec:application}
% It is interesting to note that our BiM-VFI can interpolate the frames with intended trajectories by modulating BiM. This can be easily done by setting the motion magnitude \(R\) and angle \(\Phi\) components of BiM according to a desired motion.
% \cref{fig:nonlinear} shows the visual results of interpolated frames in 3 different trajectories.
% In \cref{fig:nonlinear}, the first and last column images are source images, \(I_0\) and \(I_1\), and the second, third, and fourth columns are the interpolated frames at \(t=\) 0.25, 0.5, and 0.75, respectively.
% The first, second, and third row images are interpolated ones using the BiM with uniform maps of (\(t\), \(1.4\pi\)), (\(t\), \(\pi\)), and (\(t\), \(0.8\pi\)), respectively.
% As shown in \cref{fig:nonlinear}, our model can generate frames with various trajectories with the same input source frames by modulating the input BiM, represented in red, green, and blue lines.
% To the best of our knowledge, this is the first VFI method that has the ability to interpolate frames with desired trajectories using the same source inputs without using generative models.
% \jh{Including User study?}
% \subsection{Application 2D motion control of frame interpolation}
% \jh{Based on Case 3, showing two different curved interpolated visual results?}

%% file: sec/5_conclusion.tex
\section{Conclusion}
We proposed Bidirectional Motion field-guided VFI (BiM-VFI), which consists of (i) a distinct motion descriptor, named Bidirectional Motion Field (BiM); (ii) a BiM-guided FlowNet (BiMFN) and Context-Aware Upsampling Network (CAUN); and (iii) a Knowledge Distillation for VFI-centric Flow supervision (KDVCF).
Our BiM-VFI trained with the BiM can resolve the time-to-location ambiguity during training and interpolate clear frames by not averaging all the possible interpolation results.
In inference, our BiM-VFI can interpolate frames with very clean frames under uniform motion assumptions.
Extensive experiments have verified the effectiveness of our BiM-VFI, perceptually outperforming the recent SOTA models significantly.

% For future work, we can explore the potential of using multiple frames to predict the BiM that closely approximates the actual motion trajectories.
% % As shown in \cref{subsec:application}, our BiM-VFI can adjust the interpolation trajectory based on the input BiM.
% Employing the estimated BiM of real motion trajectories allows our BiM-VFI to interpolate the frames that are closer to actual motion trajectories.

%% file: sec/6_acknowledge.tex
\section{Acknowledgement}
This work was supported by Institute of Information \& communications Technology Planning \& Evaluation (IITP) grant funded by the Korean Government [Ministry of Science and ICT (Information and Communications Technology)] (Project Number: RS-2022-00144444, Project Title: Deep Learning Based Visual Representational Learning and Rendering of Static and Dynamic Scenes, 100\%).

%% file: sec/X_suppl.tex
\clearpage
\setcounter{page}{1}
\maketitlesupplementary
\appendix
In this supplementary material, we first provide additional details of our proposed BiM-VFI.
Especially, the detailed network structure of CAUN and SN, loss functions, implementation details, and the proof of how BiM can describe bidirectional motion distinctly are explained in \cref{sec:detail}.
Subsequently, in \cref{sec:additional_results}, we provided additional experimental results that could not be included in the main paper due to the page limitation.
In \cref{subsec:supple_quan}, pixel-centric metrics (PSNR and SSIM) in SNU-FILM-arb~\cite{guo2024generalizable}, Vimeo90K-triplet~\cite{xue2019video}, SNU-FILM~\cite{choi2020channel}, and XTest-single~\cite{sim2021xvfi}, and \(\times8\) interpolation on XTest dataset are provided.
Also, in \cref{subsec:user}, we provided the results of the user study we conducted for interpolated videos from various VFI methods.
Lastly, in \cref{subsec:supple_qual}, we provided additional qualitative comparisons on SNU-FILM-arb~\cite{guo2024generalizable} datasets.

\begin{figure*}[ht]
    \centering
    \includegraphics[width=0.8\linewidth]{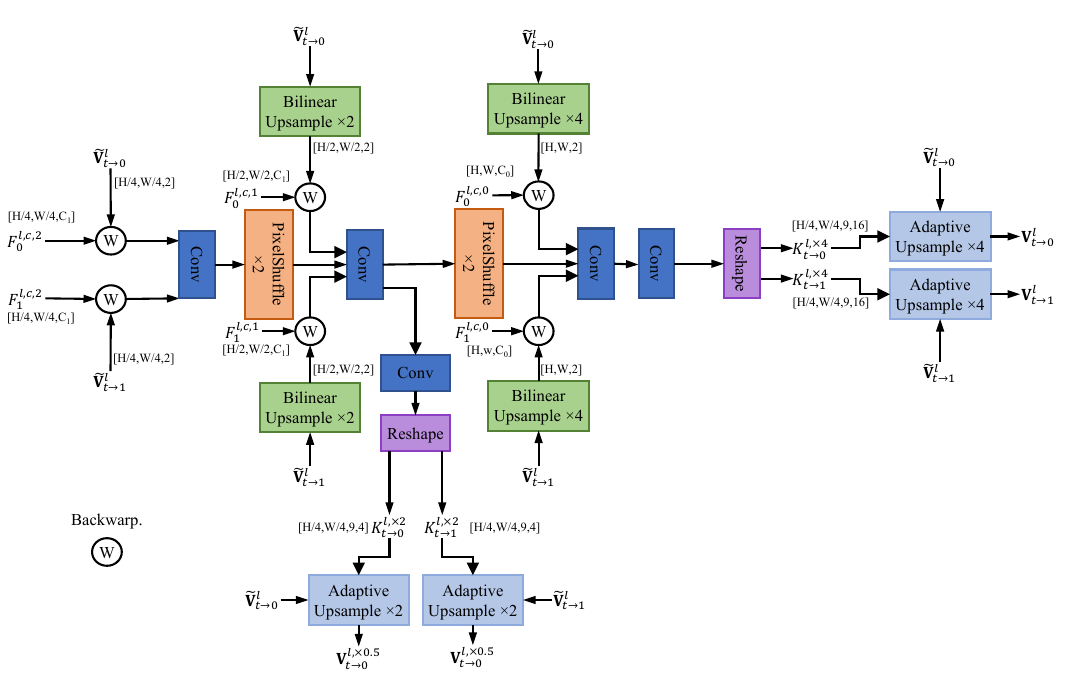}
    \caption{Detailed architecture of our Content-Aware Upsampling Network (CAUN).}
    \label{fig:caun}
\end{figure*}

\section{Additional Details}
\label{sec:detail}
\newtheorem{theorem}{Theorem}
\subsection{Structure of Content-Aware Upsampling Network (CAUN)}
\cref{fig:caun} depicts the detailed architecture of our proposed Content-Aware Upsampling Network, CAUN (\cref{subsec:caun}).
CAUN is designed to construct adaptive upsampling kernels that upsample flows while preserving high-frequency details, especially sharp boundaries and small objects.
For this, CAUN effectively utilizes and integrates multi-scale features.
Context features \(F_0^{l,c}\) and \(F_1^{l,c}\) are consists of multi-scale features \((F_0^{l,c,0}, F_0^{l,c,1}, F_0^{l,c,2})\) and \((F_1^{l,c,0}, F_1^{l,c,1}, F_1^{l,c,2})\), respectively, where \(F_i^{l,c,j}\)  is \(H/2^{j} \times W/2^{j}\)-sized context feature map of \(I_i^l\) for \(i\in\{0,1\}\) and \(j\in\{0,1,2\}\).
Note that \(F_0^{l,c,2}\) and \(F_1^{l,c,2}\) are of the same spatial sizes as \(\tilde{\textbf{V}}_{t\rightarrow0}^l\) and \(\tilde{\textbf{V}}_{t\rightarrow1}^l\).
So, the context features \(F_0^{l,c,2}\) and \(F_1^{l,c,2}\) can be directly aligned to target time \(t\) by warping via \(\tilde{\textbf{V}}_{t\rightarrow0}^l\) and \(\tilde{\textbf{V}}_{t\rightarrow1}^l\), respectively.
However, to warp \(F_0^{l,c,1}\) and \(F_1^{l,c,1}\), the two flows \(\tilde{\textbf{V}}_{t\rightarrow0}^l\) and \(\tilde{\textbf{V}}_{t\rightarrow1}^l\) must be bilinearly upsampled by a factor of 2 and their magnitudes are scaled by a factor of 2 to match with the spatial size of the features \(F_0^{l,c,1}\) and \(F_1^{l,c,1}\).
In this sense, \(\tilde{\textbf{V}}_{t\rightarrow0}^l\) and \(\tilde{\textbf{V}}_{t\rightarrow1}^l\) are further upsampled by a factor of 4, and their magnitudes are scaled by a factor of 4 to warp the features \(F_0^{l,c,0}\) and \(F_1^{l,c,0}\).
Then, the warped features are concatenated and further passed through several convolution layers and PixelShuffle layers to integrate multi-scale features.
Finally, adaptive kernels \(K_{t\rightarrow0}^{l,\times2}\), \(K_{t\rightarrow1}^{l,\times2}\), \(K_{t\rightarrow0}^{l,\times4}\), and \(K_{t\rightarrow0}^{l,\times4}\) are obtained for input with the integrated multi-scale features, where \(K_{t\rightarrow0}^{l,\times2}, K_{t\rightarrow1}^{l,\times2} \in \mathbb{R}^{\frac{H}{4} \times \frac{W}{4} \times 9 \times 4}\) and \(K_{t\rightarrow0}^{l,\times4}, K_{t\rightarrow1}^{l,\times4} \in \mathbb{R}^{\frac{H}{4} \times \frac{W}{4} \times 9 \times 16}\).
\(K_{t\rightarrow0}^{l,\times2}\) and \(K_{t\rightarrow1}^{l,\times2}\) are pixel-wise convolved with \(3\times3\) neighboring pixels of \(\tilde{\textbf{V}}_{t\rightarrow0}^l\) and \(\tilde{\textbf{V}}_{t\rightarrow1}^l\), respectively, to adaptively upsample \(\tilde{\textbf{V}}_{t\rightarrow0}^l\) and \(\tilde{\textbf{V}}_{t\rightarrow1}^l\) by a factor of 2, thus yielding \({\textbf{V}}_{t\rightarrow0}^{l,\times0.5}\) and \({\textbf{V}}_{t\rightarrow1}^{l,\times0.5}\).
Similarly, \(K_{t\rightarrow0}^{l,\times4}\) and \(K_{t\rightarrow0}^{l,\times4}\) are pixel-wise convolved with \(3\times3\) neighboring pixels of \(\tilde{\textbf{V}}_{t\rightarrow0}^l\) and \(\tilde{\textbf{V}}_{t\rightarrow1}^l\), then yielding \({\textbf{V}}_{t\rightarrow0}^{l}\) and \({\textbf{V}}_{t\rightarrow1}^{l}\), respectively, where \({\textbf{V}}_{t\rightarrow0}^{l}\) and \({\textbf{V}}_{t\rightarrow1}^{l}\) are of the same sizes as those of the source images at \(l\)-th level, \(I_0^l\) and \(I_1^l\).
\({\textbf{V}}_{t\rightarrow0}^{l,\times0.5}\), \({\textbf{V}}_{t\rightarrow1}^{l,\times0.5}\), \({\textbf{V}}_{t\rightarrow0}^{l}\), and \({\textbf{V}}_{t\rightarrow1}^{l}\) are further utilized in Synthesis Network (SN) to warp the source images and their context features with more precise flows.
\subsection{Structure of Synthesis Network (SN)}
\begin{figure}[!htbp]
    \centering
    \includegraphics[width=\linewidth]{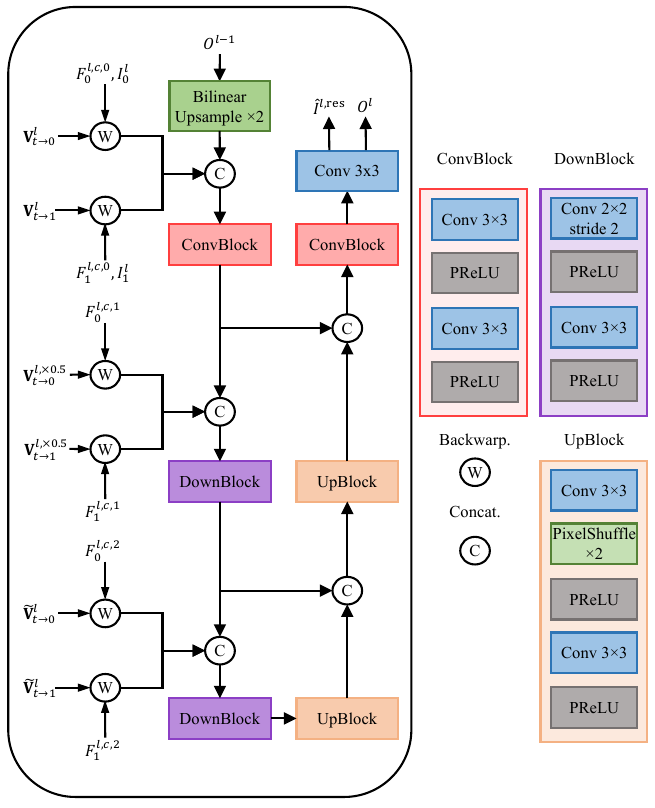}
    \caption{Detailed architecture of our Synthesis Network (SN).}
    \label{fig:sn}
\end{figure}

We employed a simple U-Net~\cite{ronneberger2015u} for our Synthesis Network (SN) as depicted in \cref{fig:sn}.
The multi-scale flows from CAUN, which include \((\tilde{\textbf{V}}_{t\rightarrow0}^{l}, \tilde{\textbf{V}}_{t\rightarrow1}^{l})\), \(({\textbf{V}}_{t\rightarrow1}^{0, \times0.5}, {\textbf{V}}_{t\rightarrow1}^{l, \times 0.5})\) and \(({\textbf{V}}_{t\rightarrow0}^{l}, {\textbf{V}}_{t\rightarrow1}^{l})\), are used to warp the multi-scale context features \((F_0^{l,c,2}, F_1^{l,c,2})\), \((F_0^{l,c,1}, F_1^{l,c,1})\), and \((F_0^{l,c,0}, F_1^{l,c,0})\).
As depicted in \cref{fig:sn}, the warped multi-scale context features are passed through the U-Net, finally yielding a blending mask \(O^l\) and a residual image \(\hat{I}^{l,\text{res}}_t\) at \(l\)-th level.
The resulting \(\hat{I}_t^{l, \text{res}}\) and \(O^{l}\) from the U-Net are employed to construct final interpolation result at \(l\)-th level, \(\hat{I}_t^l\), as follow:
\begin{equation}
\begin{split}
    \hat{I}_{t}^{l} = &\text{ bw}(I_0^l, {\textbf{V}}_{t\rightarrow0}^{l}) * \sigma(O^{l}) + \\
    &\text{ bw}(I_1^l, {\textbf{V}}_{t\rightarrow1}^{l}) * (1 - \sigma(O^{l})) + \hat{I}_t^{l, \text{res}},
\end{split}
\end{equation}
where \(\text{bw}(\cdot, \cdot)\) is a backward warping function and \(\sigma(\cdot)\) is a sigmoid function.

\subsection{Loss Functions}

Our training objectives consist of student loss \(\mathcal{L}_{\mathcal{P}_{\mathcal{S}}}\), and teacher loss \(\mathcal{L}_{\mathcal{P}_{\mathcal{T}}}\). The teacher loss will be discussed first, and followed by the student loss. To supervise photometric reconstruction of the teacher process, the charbonnier loss~\cite{charbonnier1994two} \(\mathcal{L}_{\text{char}}\) and the census loss~\cite{meister2018unflow} \(\mathcal{L}_{\text{css}}\) are used as follow:
\begin{equation}
    \begin{split}
        &\mathcal{L}_{\text{char},\mathcal{P}_{\mathcal{T}}}^l = \lambda_{\text{char},\mathcal{P}_{\mathcal{T}}}\mathcal{L}_{\text{char}}(\hat{I}^{l,\mathcal{P}_{\mathcal{T}}}, I^{l}), \\
        &\mathcal{L}_{\text{css},\mathcal{P}_{\mathcal{T}}}^l = \lambda_{\text{css},\mathcal{P}_{\mathcal{T}}}\mathcal{L}_{\text{css}}(\hat{I}^{l,\mathcal{P}_{\mathcal{T}}}, I^{l}), \\
        &\mathcal{L}_{\text{pho},\mathcal{P}_{\mathcal{T}}}^l = \mathcal{L}_{\text{char},\mathcal{P}_{\mathcal{T}}}^l + \mathcal{L}_{\text{css},\mathcal{P}_{\mathcal{T}}}^l,
    \end{split}
\end{equation}
where \(l\) is the current pyramid level, \(\lambda_{\text{char},\mathcal{P}_{\mathcal{T}}}\) and \(\lambda_{\text{css},\mathcal{P}_{\mathcal{T}}}\) are weights for each loss.
Furthermore, the first-order edge-aware smoothness loss~\cite{jonschkowski2020matters} \(\mathcal{L}_{\text{s1}}\) is used to ensure smooth teacher flows excluding object boundaries, and the regularization loss \(\mathcal{L}_{\text{reg}}\) is utilized to force \(V_{t\rightarrow t|0t}^{l, \mathcal{P}_{\mathcal{T}}}\) and \(V_{t\rightarrow t|t1}^{l, \mathcal{P}_{\mathcal{T}}}\) to be uniform vector fields of all zeros, where the input BiM of teacher process (\cref{eq:teacher_bim0}, \cref{eq:teacher_bim1}) is used to guide the two flows to be zero flows:
\begin{equation}
    \begin{split}
        &\mathcal{L}_{\text{s1}}^l = \lambda_{\text{s1}}(\mathcal{L}_{\text{s1}}(V_{t\rightarrow0}^{l, \mathcal{P}_{\mathcal{T}}}) + \mathcal{L}_{\text{s1}}(V_{t\rightarrow1}^{l, \mathcal{P}_{\mathcal{T}}})), \\
        &\mathcal{L}_\text{reg}^l = \lambda_\text{reg}(\mathcal{L}_2(V_{t\rightarrow t|0t}^{l, \mathcal{P}_{\mathcal{T}}}) + \mathcal{L}_2(V_{t\rightarrow t|t1}^{l, \mathcal{P}_{\mathcal{T}}})), \\
        &\mathcal{L}_{\text{flo},\mathcal{P}_{\mathcal{T}}}^l = \mathcal{L}_{\text{s1}}^l + \mathcal{L}_\text{reg}^l,
    \end{split}
\end{equation}
where \(\lambda_{\text{s1}}\) and \(\lambda_\text{reg}\) are weights for each loss.

\begin{table*}[ht]
\centering

{\small
\setlength\tabcolsep{3pt}
\begin{tabular}{l | c c | c c | c c | c c c c c}
\multirow{3}{*}{Methods} & \multicolumn{6}{c}{SNU-FILM-arb} & \multicolumn{5}{c}{XTest} \\
 & \multicolumn{2}{c}{medium} & \multicolumn{2}{c}{hard} & \multicolumn{2}{c}{extreme} &  \multicolumn{5}{c}{\(\times\)8}  \\
& psnr & ssim & psnr & ssim & psnr & ssim & psnr & ssim & lpips & stlpips & niqe \\\hline\hline
RIFE~\cite{huang2022real} & 36.31 & \underline{0.981} & 31.86 & \underline{0.952} & 27.20 & \underline{0.895} & 30.58 & 0.904 & 0.153 & 0.114 & 7.393 \\
IFRNet~\cite{kong2022ifrnet} & 34.82 & 0.976 & 31.11 & 0.947 & 26.29 & 0.882 & 26.36 & 0.826 & 0.198 & 0.147 & \textbf{5.842} \\
M2M-PWC~\cite{hu2022many} & 36.54 & \textbf{0.982} & \underline{31.92} & 0.951 & 27.13 & 0.892 & \underline{30.81} & \underline{0.912} & \underline{0.080} & \underline{0.047} & 6.521 \\
AMT-S~\cite{li2023amt} & 34.42 & 0.974 & 30.98 & 0.947 & 26.42 & 0.887 & 28.16 & 0.873 & 0.187 & 0.134 & 7.082 \\
UPRNet~\cite{jin2023unified} & \underline{36.70} & \textbf{0.982} & 31.9 & 0.951 & 27.08 & 0.893 & 30.50 & 0.905 & 0.093 & 0.058 & \underline{6.148} \\
EMA-VFI~\cite{zhou2023exploring} & \textbf{36.85} & \textbf{0.982} & \textbf{32.7} & \textbf{0.957} & \textbf{28.15} & \textbf{0.906} & \textbf{31.36} & \textbf{0.914} & 0.165 & 0.130 & 7.77 \\
\hline
RIFE[D,R]~\cite{zhong2023clearer} & 36.17 & 0.981 & 31.59 & 0.949 & 27.05 & 0.891 & 26.93 & 0.839 & 0.232 & 0.169 & 6.477 \\
IFRNet[D,R]~\cite{zhong2023clearer} & 35.92 & 0.981 & 31.18 & 0.947 & 26.54 & 0.886 & 28.76 & 0.891 & 0.147 & 0.096 & 7.054 \\
AMT-S[D,R]~\cite{zhong2023clearer} & 34.78 & 0.978 & 30.48 & 0.944 & 26.15 & 0.886 & 29.27 & 0.886 & 0.098 & 0.055 & 6.409 \\
EMA-VFI[D,R]~\cite{zhong2023clearer} & 35.75 & 0.980 & 31.02 & 0.946 & 26.37 & 0.885 & 25.75 & 0.833 & 0.258 & 0.192 & 6.928 \\
\hline
ours & 36.57 & \textbf{0.982} & \underline{31.92} & 0.949 & \underline{27.22} & 0.891 & 30.80 & \textbf{0.914} & \textbf{0.068} & \textbf{0.045} & 6.449 \\

\end{tabular}
}
\caption{Additional quantitative comparisons on arbitrary time interpolation datasets.}
\label{tab:supple_arb}
\end{table*}
The photometric loss for the student process is constructed in the same manner as the teacher process, which is given by:
\begin{equation}
    \begin{split}
        &\mathcal{L}_{\text{char},\mathcal{P}_{\mathcal{S}}}^l = \lambda_{\text{char},\mathcal{P}_{\mathcal{S}}}\mathcal{L}_{\text{char}}(\hat{I}^{l,\mathcal{P}_{\mathcal{S}}}, I^{l}), \\
        &\mathcal{L}_{\text{css},\mathcal{P}_{\mathcal{S}}}^l = \lambda_{\text{css},\mathcal{P}_{\mathcal{S}}}\mathcal{L}_{\text{css}}(\hat{I}^{l,\mathcal{P}_{\mathcal{S}}}, I^{l}), \\
        &\mathcal{L}_{\text{pho},\mathcal{P}_{\mathcal{S}}}^l = \mathcal{L}_{\text{char},\mathcal{P}_{\mathcal{S}}}^l + \mathcal{L}_{\text{css},\mathcal{P}_{\mathcal{S}}}^l,
    \end{split}
\end{equation}
where \(\lambda_{\text{char},\mathcal{P}_{\mathcal{S}}}\) and \(\lambda_{\text{css},\mathcal{P}_{\mathcal{S}}}\) are the weights for their respective losses. The flows of the student process will be supervised by a flow distillation loss that enforces the flows of the student process to get closer to those of the teacher process, which is given by:
\begin{equation}
    \begin{split}
        \mathcal{L}_{\text{flo},\mathcal{P}_{\mathcal{S}}}^l = \lambda_{\text{distill}}(\mathcal{L}_2(V_{t\rightarrow 0}^{l,\mathcal{P}_{\mathcal{S}}} - \text{sg}(V_{t\rightarrow 0}^{l,\mathcal{P}_{\mathcal{T}}} )) \\
        + \mathcal{L}_2(V_{t\rightarrow 1}^{l, \mathcal{P}_{\mathcal{S}}} - \text{sg}(V_{t\rightarrow 1}^{\mathcal{P}_{\mathcal{T}}}))),
    \end{split}
\end{equation}
where \(\lambda_\text{distill}\) is a weighting factor, and \(\text{sg}(\cdot)\) is a stop gradient function that is used to force the gradients to be only activated for the student process. The overall loss for our BiM-VFI with KDVCF is defined as:
\begin{equation}
\begin{split}
    \mathcal{L} = \sum_{l=0}^{L-1} &\gamma_\text{pho}^{l}(\mathcal{L}_{\text{pho},\mathcal{P}_{\mathcal{T}}}^l + \mathcal{L}_{\text{pho}, \mathcal{P}_{\mathcal{S}}}) + \\
    &\gamma_\text{flo}^{l}(\mathcal{L}_{\text{flo},\mathcal{P}_{\mathcal{T}}} + \mathcal{L}_{\text{flo},\mathcal{P}_{\mathcal{S}}}),
\end{split}
\end{equation}
% \begin{equation}
%     \mathcal{L}_{tea} = \sum_{l=0}^{L-1}\mathcal{L}_{\text{char}}(\hat{I}^{l,tea}, I^{l}) + \mathcal{L}_{\text{css}}(\hat{I}^{l,tea}, I^{l}) + \mathcal{L}_{smooth1} + (V_{t\rightarrow0}^{l, tea}) + \mathcal{L}_{smooth1}(V_{t\rightarrow1}^{l, tea}) + \mathcal{L}_{reg}(V_{t\rightarrow t|0t}^{l, tea}) + \mathcal{L}_{reg}(V_{t\rightarrow t|t1}^{l, tea})
% % \end{equation}
% \begin{equation}
%     \begin{split}
%         \mathcal{L}_{tea} = \sum_{l=0}^{L-1}\gamma_{pho,tea}^{l}(\lambda_{\text{char},tea}\mathcal{L}_{\text{char},tea} + \lambda_{\text{css},tea}\mathcal{L}_{\text{css}, tea}) + \\ \gamma_{flo,tea}^{L-l}(\lambda_{smooth1}\mathcal{L}_{smooth1} + \lambda_{reg}\mathcal{L}_{reg}),
%     \end{split}
% \end{equation}
where \(L\) is the total number of pyramid levels used in training, and \(\gamma_\text{pho}\), and \(\gamma_\text{flo}\) are exponential weights for the photometric loss and the flow-centric loss, respectively, which are employed to weigh more supervision on larger-sized image resolutions.

\begin{table*}[ht]
\centering

{\small
\setlength\tabcolsep{3pt}
\begin{tabular}{l | c c | c c | c c | c c | c c | c c | c c}
\multirow{3}{*}{Methods} & \multicolumn{2}{c}{\multirow{2}{*}{Vimeo 90K-triplet}} & \multicolumn{8}{c}{SNU-FILM} & \multicolumn{2}{c}{{XTest}} & \multicolumn{2}{c}{Complexity} \\
 &  &  & \multicolumn{2}{c}{easy} & \multicolumn{2}{c}{medium} & \multicolumn{2}{c}{hard} & \multicolumn{2}{c}{extreme} & \multicolumn{2}{c}{{single}} & FLOPs & Params \\
 & psnr & ssim & psnr & ssim & psnr & ssim & psnr & ssim & psnr & ssim & psnr & ssim & (T) & (M) \\\hline\hline
AMT-G~\cite{li2023amt} & \underline{36.53} & \textbf{0.982} & 39.88 & \textbf{0.991} & \underline{36.12} & \textbf{0.981} & 30.78 & \textbf{0.981} & 25.43 & \underline{0.865} & \textbf{30.34} & \textbf{0.904} & 2.07 & 30.6 \\
M2M-PWC~\cite{hu2022many} & 35.49 & 0.978 & 39.66 & \textbf{0.991} & 35.74 & 0.980 & 30.32 & \underline{0.980} & 25.07 & 0.863 & \textbf{30.81} & 0.900 & \underline{0.26} & 7.6 \\
UPRNet~\cite{jin2023unified} & 36.42 & \textbf{0.982} & \textbf{40.44} & \textbf{0.991} & \textbf{36.29} & 0.980 & \underline{30.86} & 0.938 & \underline{25.63} & 0.864 & 30.27 & 0.897 & 1.59 & 6.6 \\
RIFE~\cite{huang2022real} & 35.61 & 0.978 & 40.02 & \underline{0.990} & 35.72 & 0.979 & 30.08 & 0.933 & 24.84 & 0.853 & 23.57 & 0.778 & 0.20 & 9.8 \\
XVFI~\cite{sim2021xvfi} & 33.99 & 0.968 & 38.37 & 0.987 & 34.42 & 0.973 & 29.52 & 0.928 & 24.88 & 0.854 & 28.96 & 0.887 & \textbf{0.21} & \textbf{5.7} \\
IFRNet~\cite{kong2022ifrnet} & 36.16 & \underline{0.980} & \underline{40.10} & \textbf{0.991} & \underline{36.12} & 0.978 & 30.63 & 0.936 & 25.27 & 0.861 & 27.53 & 0.847 & 0.79 & 19.7 \\
EMA-VFI~\cite{zhou2023exploring} & \textbf{36.64} & \textbf{0.982} & 39.98 & \textbf{0.991} & 36.09 & \underline{0.980} & \textbf{30.94} & 0.939 & \textbf{25.69} & \textbf{0.866} & 29.89 & 0.896 & 0.91 & 66.0 \\\hline
Ours & 35.01 & 0.977 & 40.09 & \underline{0.990} & 35.89 & 0.979 & 30.54 & 0.935 & 25.33 & 0.860 & 29.90 & \underline{0.901} & 0.91 & \underline{6.0} \\

\end{tabular}
}
\caption{Additional quantitative comparisons on fixed time interpolation datasets and the complexity of SOTA models.}
\label{tab:supple_fixed}
\end{table*}
\subsection{Implementation details}
We trained our BiM-VFI with a training split of Vimeo90k septuplet datasets~\cite{xue2019video}.
We randomly crop the images to a resolution of 256 \(\times\) 256, flip horizontally and vertically, rotate, reverse temporally, and permute the color channels to augment the training data. We set the batch size to 32, and train the model for 400 epochs with an initial learning rate of 1 \(\times 10^{-4}\).
We gradually decay the initial learning rate using a Cosine annealing scheduler~\cite{loshchilov2016sgdr} and optimize our model using the AdamW optimizer~\cite{loshchilov2017decoupled}.
Also, because the architecture of our BiM-VFI is based on a recurrent pyramid architecture, we employed resolution-aware adaptation for the pyramid hierarchy depth proposed by Jin \etal~\cite{jin2023unified}.
For training on Vimeo90K, we used 3 pyramid levels, while 5 pyramid levels are used for SNU-FILM~\cite{choi2020channel} and SNU-FILM-arb~\cite{guo2024generalizable}, and 7 pyramid levels for Xtest~\cite{sim2021xvfi}.
As mentioned, our proposed KDVCF computes the BiM during training, and for inference time, the BiM is represented according to \cref{eq:unifrom_bim} corresponding to a uniform motion scenario.

\begin{figure}[t]
\centering
    \begin{subfigure}[b]{\linewidth}
    \centering
        \includegraphics{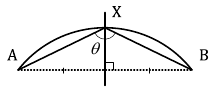}
        \caption{Unique intersection of loci in case of \(k=1\)}
        \label{subfig:bisector}
        
    \end{subfigure}
    \begin{subfigure}[b]{\linewidth}
        \centering
        \includegraphics{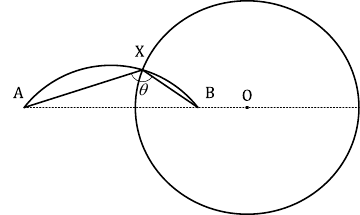}
        \caption{Unique intersection of loci in case of \(k\neq1\)}
        \label{subfig:circle}
        
    \end{subfigure}
    \caption{Visualization of intersection by two loci.}
    \label{fig:bim_proof}
\end{figure}
\subsection{Distinct Description of BiM}
As discussed in \cref{subsec:motion field} of the main paper, we proposed BiM as a distinct motion descriptor for non-uniform motions, including accelerations, decelerations, and changing directions.
To ensure the distinct descriptive power of BiM, we provide a mathematical analysis of how our BiM can explain the position of the intermediate pixel between given two corresponding pixels.
\begin{theorem}
Let \(A\) and \(B\) be two fixed points, and let \(k\) be a positive real number.
The point \(X\) such that the distance ratio \(\frac{\overline{AX}}{\overline{BX}} = k\) and the angle \(\angle AXB = \theta\) is unique.
\end{theorem}

\begin{proof}

We start by describing the locus of points \(X'\) where \(\angle AX'B = \theta\).
This locus forms an arc \(\overset{\frown}{AB}\) where any point \(X'\) on the arc \(\overset{\frown}{AB}\) satisfying \(\angle AX'B = \theta\). 

The locus of points \(X''\) where \(\frac{\overline{AX''}}{\overline{BX''}} = k\) varies inshape depending on the value of \(k\). 

I) If \(k=1\) (\cref{subfig:bisector}), this locus forms a perpendicular bisector of \(\overline{AB}\).
In this case, the intersection of the arc \(\overset{\frown}{AB}\) and the perpendicular bisector of \(\overline{AB}\) is unique, thus the point satisfying the distance ratio \(\frac{\overline{AX}}{\overline{BX}} = k\) and the angle \(\angle AXB = \theta\) is unique. 

II) If \(k \neq 1\) (\cref{subfig:circle}), this locus forms an Apollonian circle~\cite{mittal2012making}.
In this case, if \(k>1\), point \(A\) is outside the circle, and point \(B\) is inside the circle. Conversely, if \(k<1\), point \(B\) is outside the circle, and point \(A\) is inside the circle.
In any case, the resulting circle has only one intersection with the arc \(\overset{\frown}{AB}\), thus the point satisfying the distance ratio \(\frac{\overline{AX}}{\overline{BX}} = k\) and the angle \(\angle AXB = \theta\) is unique. 

By I) and II), for given two fixed points \(A\) and \(B\), a positive real number \(k\), it is concluded that the point \(X\) satisfying the distance ratio \(\frac{\overline{AX}}{\overline{BX}} = k\) and the angle \(\angle AXB = \theta\) is unique.
\end{proof}

\section{Additional Experimental Results}
\label{sec:additional_results}
\subsection{Quantitative Results}
\label{subsec:supple_quan}
We provided pixel-centric metrics (PSNR and SSIM) measured on SNU-FILM-arb~\cite{guo2024generalizable} datasets and additional arbitrary time interpolation on XTest~\cite{sim2021xvfi} to interpolate \(\times8\) frames, which are tabulated in \cref{tab:supple_arb}.
As discussed in \cref{subsec:quan} of the main paper, while our BiM-VFI underperforms in pixel-centric metrics, it consistently outperforms the other SOTA methods on XTest~\cite{sim2021xvfi} \(\times8\) interpolation in terms of perceptual metrics, such as LPIPS and STLPIPS.

In \cref{tab:supple_fixed}, we also provided pixel-centric metrics measured on fixed-time datasets (Vimeo 90K-triplet~\cite{xue2019video}, SNU-FILM~\cite{choi2020channel}, and XTest~\cite{sim2021xvfi} single) and complexity comparisons between other SOTA methods.

\subsection{Computational complexity}
We additionally provide below \#'s of parameters and FLOPs on each component for 256\(\times\)256-sized images.

\begin{table}[ht]
\centering
{\fontsize{7pt}{8pt}\selectfont
\setlength\tabcolsep{3pt}
\begin{tabular}{l |  c c  c  c  c | c}
\hline
 & MFE & CFE & BiMFN & CAUN & SN & Total \\
\hline
\#Params(M) & 0.58 & 0.58 & 3.41 & 0.61 & 1.7 & 6.88 \\
\hline
FLOPs(G) & 12.02 & 12.02 & 18.81 & 11 & 24.64 & 78.49 \\

\hline
\end{tabular}
}
\end{table}

We measured FLOPs for interpolating 1280\(\times\)720-sized source images and the total parameters used in the methods.
As shown in \cref{tab:supple_fixed}, our BiM-VFI effectively reduced the number of parameters by employing a recurrent pyramid architecture, while having moderate computational complexity among the other SOTA methods in terms of FLOPs.

\begin{table}[ht]
\centering

{\fontsize{7pt}{8pt}\selectfont
\setlength\tabcolsep{3pt}
\begin{tabular}{l | c  c  c  c | c}
\hline
 & GIMM-VFI-R & EMA-VFI & UPR & AMT & Ours \\
\hline
\#Params(M) & 19.73 & 65.66 & 6.56 & 30.64 & \textbf{6.88} \\
\hline
FLOPs(G) & 9187 & 1714 & 1228 & 2395 & \textbf{1177} \\
\hline
Runtime(ms) & 494 & 104 & \textbf{53} & 183 & 151 \\
\hline
\end{tabular}
}

\end{table}

\subsection{Additional ablation study}
We provide below the detection performance of small objects and object boundaries with and without CAUN module. As shown, the CAUN can help capture well small toes (top) and detect tight boundaries of the windmill blade (bottom), while failing without it. 

\begin{figure}[ht]
    \centering
    \includegraphics[width=\linewidth]{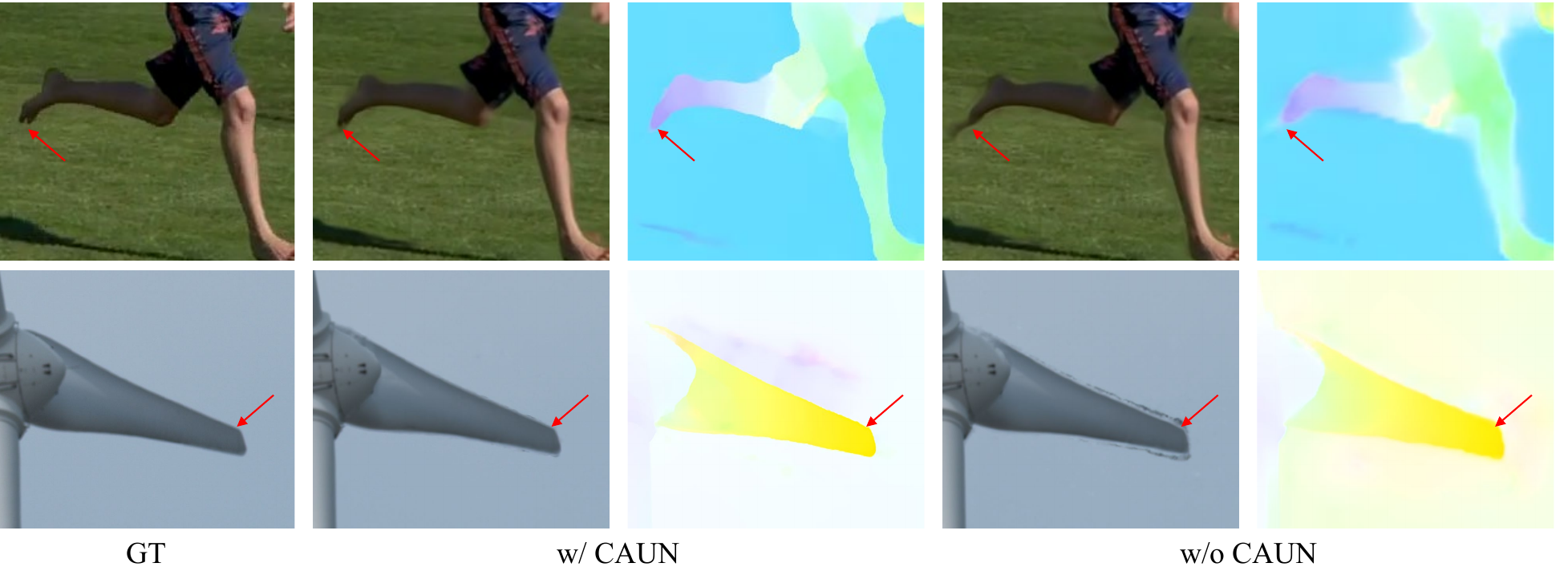}
    \label{fig:caun__ablation}
\end{figure}
\noindent Also, as mentioned in \textit{Suppl.}, while our KDVCF increases the training time from 2.5 to 4 days using 4 A100 GPUs due to the \(\mathcal{P}_{\mathcal{T}}\) process, it does not increase the inference time.

\noindent We also compared the \#'s of parameters, FLOPs, and runtime (measured on an A100 GPU with 1280\(\times\)768-sized images) in the below table. It can be noted that our BiM-VFI has a low number of parameters, showing moderate runtime.

\begin{figure}[t]
    \centering
    \includegraphics[width=0.8\linewidth]{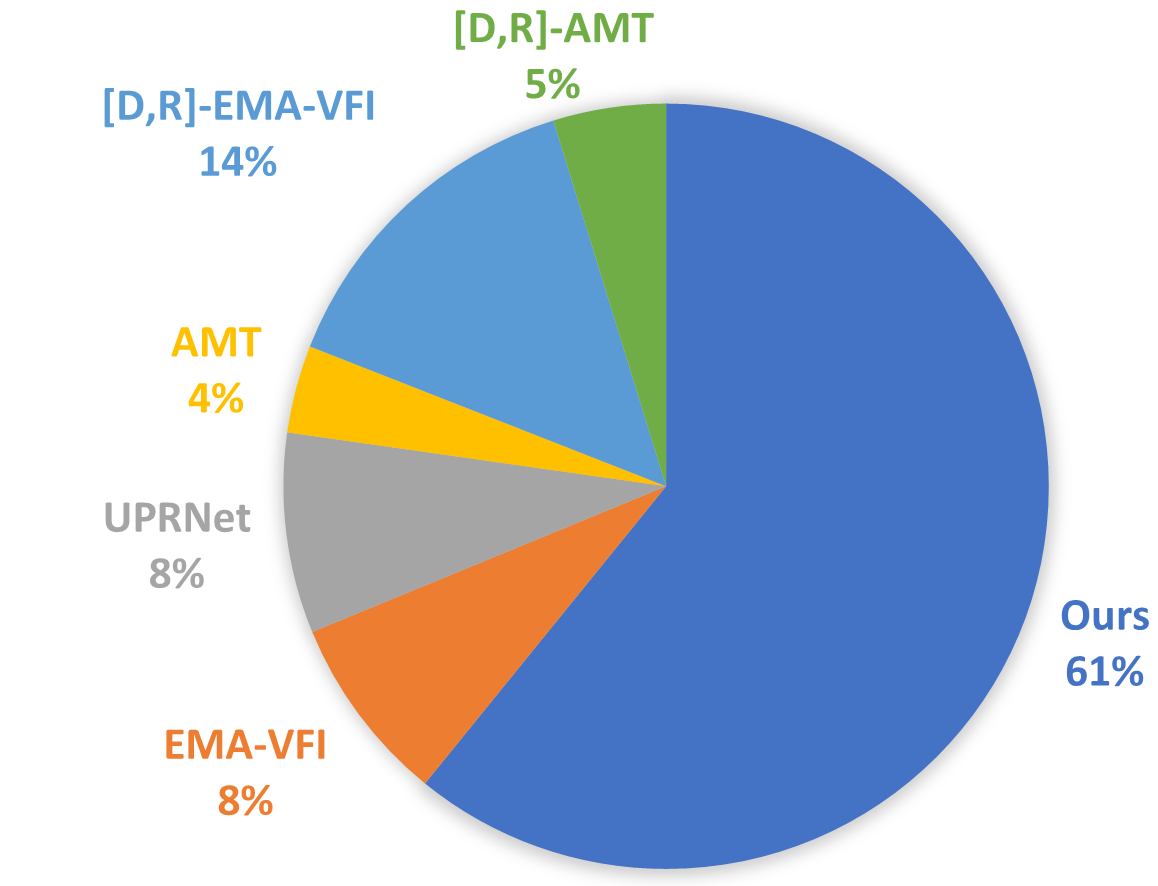}
    \caption{Preference of interpolated videos between our BiM-VFI and the other SOTA models measured by user study.}
    \label{fig:user}
\end{figure}

\subsection{User Study}
\label{subsec:user}
We conducted a user study to show that our BiM-VFI with uniform motion BiM perceptually outperforms other SOTA methods.
21 participants were asked to choose the best-interpolated videos among AMT-S~\cite{li2023amt}, UPRNet~\cite{jin2023unified}, EMA-VFI~\cite{zhou2023exploring}, [D,R]-AMT-S, and [D,R]-EMA-VFI, where [D,R] indicates that distance indexing and iterative reference-based estimation, proposed by Zhong \etal~\cite{zhong2023clearer}, are plugged into the method.
We used 9 test videos for blind subjective tests where the six 8\(\times\)-interpolated videos for the six VFI methods including our BiM-VFI are displayed simultaneously on the same screens for each test video.
In order to remove any subjective bias to specific VFI methods, the six 8\(\times\)-interpolated videos for each test video are randomly ordered and presented to the participants in the blind subjective test.

As shown in \cref{fig:user}, our BiM-VFI dominantly outperforms the other SOTA methods in the subjective tests, by 61\% preference against the other six VFI methods. 
%The videos used in our user study are also included in our supplementary material.

\subsection{Qualitative Results}
\label{subsec:supple_qual}
We provided additional qualitative comparisons with the SOTA methods in SNU-FILM-arb~\cite{guo2024generalizable} extreme datasets.

\section{Limitation}
Our KDVCF requires approximately twice the training time compared to training solely with the student process, as both the teacher and student processes are trained simultaneously.
However, the model trained with KDVCF demonstrated its effectiveness in perceptual metrics compared to models supervised with pre-trained flow models or without flow supervision.
It is also noteworthy that only the student process remains during inference, so the inference runtime is the same as that of models trained without KDVCF.

\clearpage
\begin{figure*}
    \centering
    \includegraphics[width=\linewidth]{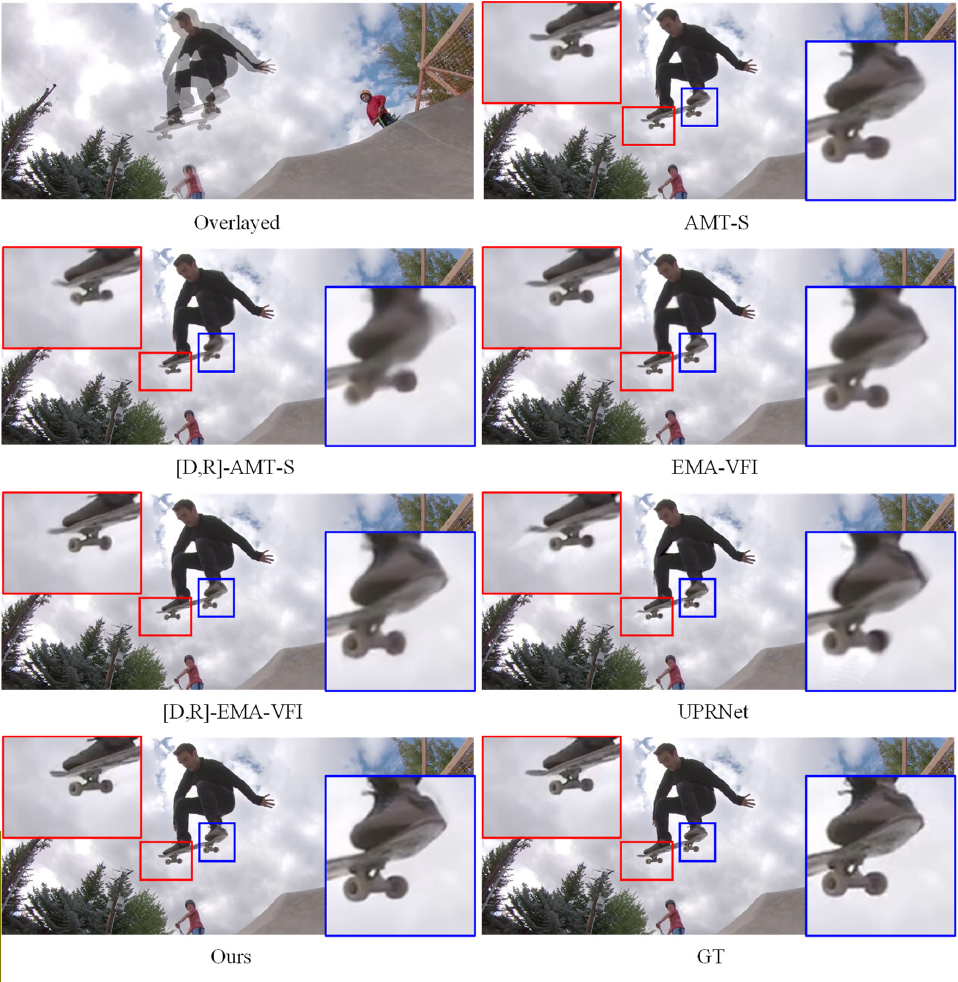}
    \caption{Additional qualitative comparisons on SNU-FILM-arb~\cite{guo2024generalizable} extreme datasets.}
    \label{fig:supple_snufilm1}
\end{figure*}

\clearpage
\begin{figure*}
    \centering
    \includegraphics[width=\linewidth]{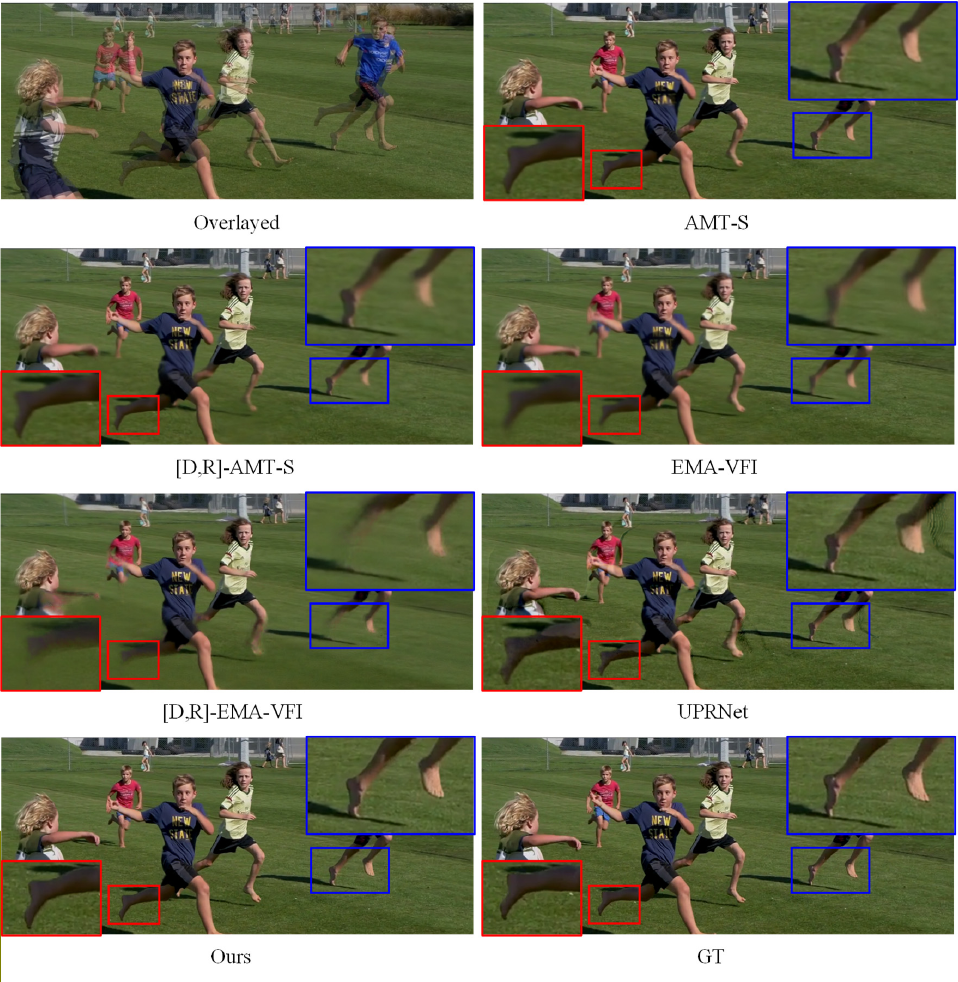}
    \caption{Additional qualitative comparisons on SNU-FILM-arb~\cite{guo2024generalizable} extreme datasets.}
    \label{fig:supple_snufilm2}
\end{figure*}

\clearpage
\begin{figure*}
    \centering
    \includegraphics[width=\linewidth]{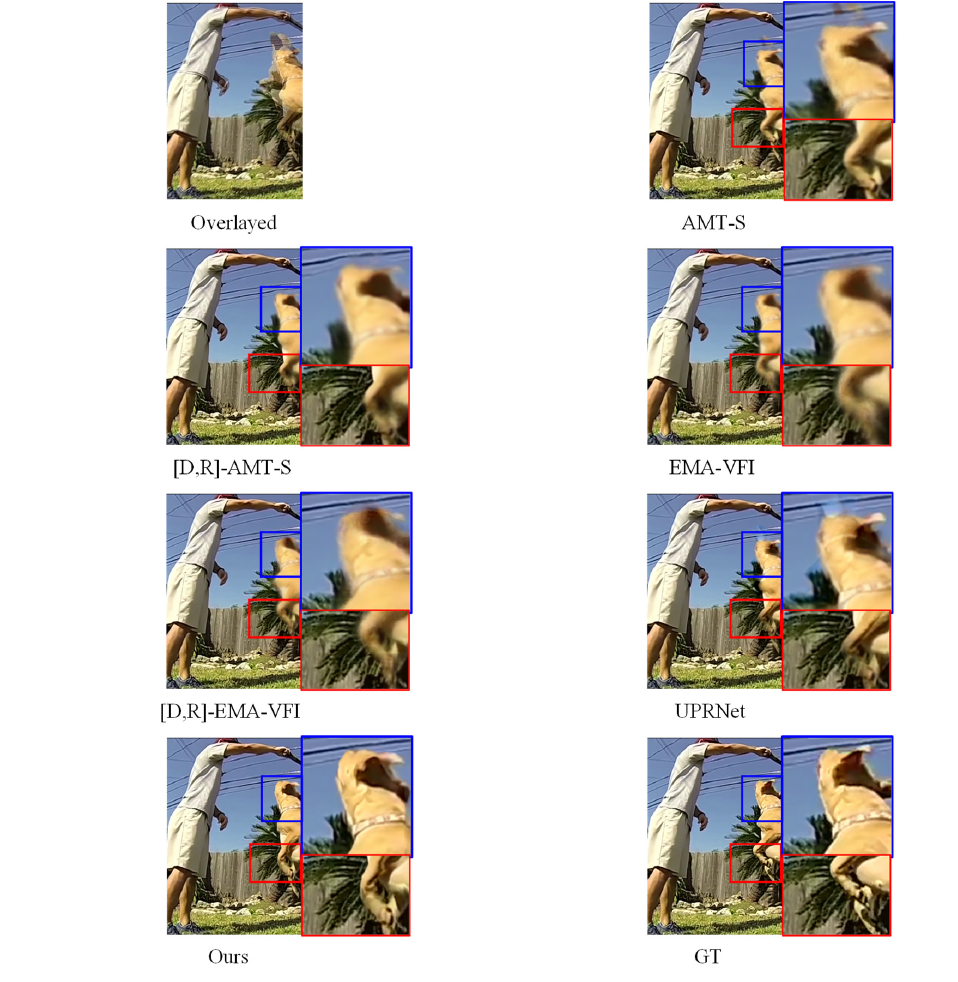}
    \caption{Additional qualitative comparisons on SNU-FILM-arb~\cite{guo2024generalizable} extreme datasets.}
    \label{fig:supple_snufilm3}
\end{figure*}

\clearpage
\begin{figure*}
    \centering
    \includegraphics[width=\linewidth]{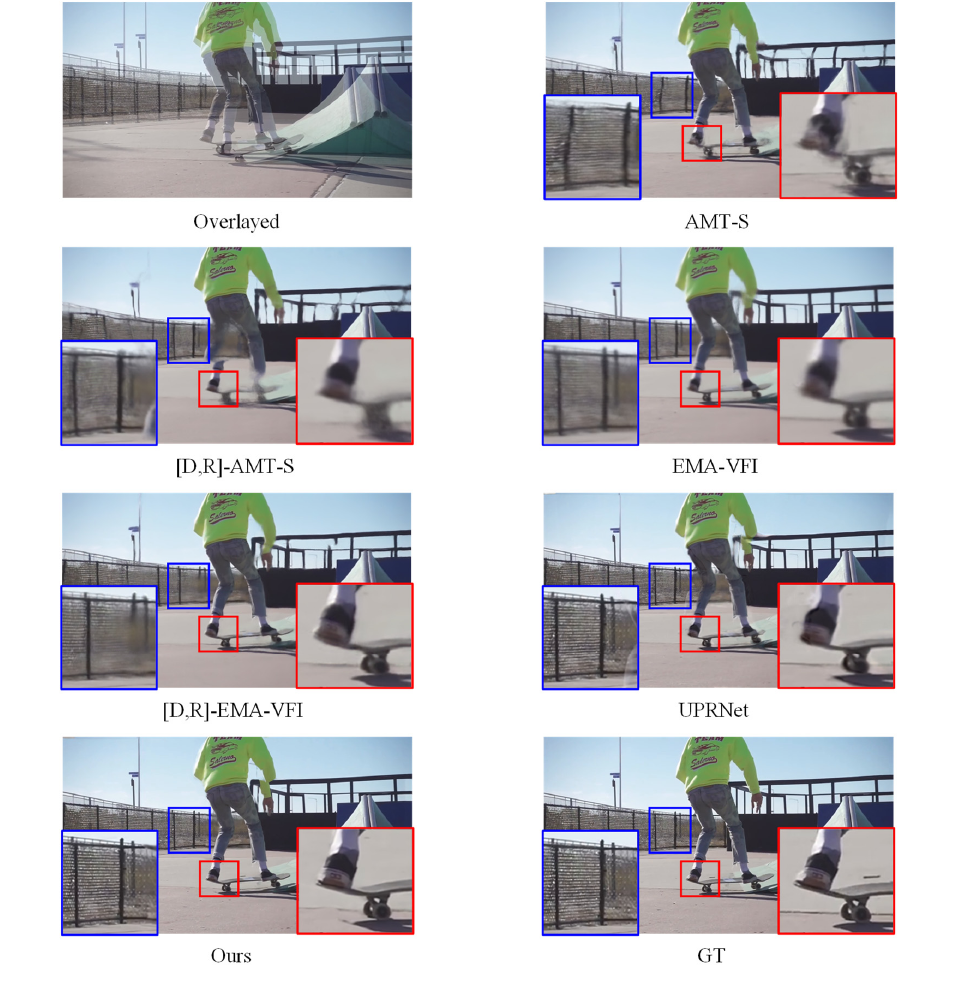}
    \caption{Additional qualitative comparisons on SNU-FILM-arb~\cite{guo2024generalizable} extreme datasets.}
    \label{fig:supple_snufilm4}
\end{figure*}

\clearpage